\documentclass[11pt, letterpaper]{article}
\usepackage[margin=1in]{geometry}
\usepackage[utf8]{inputenc} 

\usepackage{booktabs} 
\usepackage[table,xcdraw]{xcolor}
\usepackage{bbm}
\usepackage{algorithm}
\usepackage{algpseudocode}
\usepackage{mathtools}
\usepackage{macros_ms} 

\newcommand{\EE}{\mathbb{E}}
\newcommand{\RR}{\mathbb{R}}
\newcommand{\DD}{\mathcal{D}}
\newcommand{\algparbox}[1]{\parbox[t]{\dimexpr\linewidth-\algorithmicindent}{#1\strut}}
\def\argmax{\mathop{\rm arg\,max}}%
\def\argmin{\mathop{\rm arg\,min}}%

\usepackage[labelfont=bf,font={footnotesize,stretch=0.90}]{caption}
\usepackage[labelformat=simple]{subcaption}

\usepackage{ragged2e}
\usepackage{natbib}
\usepackage[english]{babel}
\usepackage{amsfonts}
\usepackage{amsmath, amsthm, amssymb}
\usepackage{ifthen}
\usepackage{url}
\usepackage{graphicx}
\usepackage{color}
\usepackage{array}
\usepackage{afterpage}

\newtheorem{proposition}{Proposition}
\newtheorem{theorem}{Theorem}
\newtheorem{lemma}{Lemma}
\newtheorem{corollary}{Corollary}

\usepackage{authblk}
\title{Learning to Be Fair: A Consequentialist Approach\\ to Equitable Decision-Making}
\author[1]{Alex Chohlas-Wood}
\author[1]{Madison Coots}
\author[2]{Henry Zhu}
\author[2]{\\Emma Brunskill}
\author[1]{Sharad Goel}
\affil[1]{Kennedy School, Harvard University}
\affil[2]{Computer Science, Stanford University}
\date{}
\begin{document}
\maketitle 

\begin{abstract}
\noindent
In an attempt to make algorithms \emph{fair}, the machine learning literature has
largely focused on equalizing decisions, outcomes, or error rates across race or gender groups.
To illustrate, consider a hypothetical government rideshare program that provides transportation assistance to low-income people with upcoming court dates.
Following this literature, one might allocate rides to those with the highest estimated treatment effect per dollar, 
while constraining spending to be equal across race groups.
That approach, however, ignores the downstream consequences of such constraints, and, as a result, can induce unexpected harms.
For instance, if one demographic group lives farther from court, enforcing equal spending would necessarily mean fewer total rides provided, and potentially more people penalized for missing court.
Here we present an alternative framework for designing equitable algorithms that foregrounds the consequences of decisions.
In our approach, one first elicits stakeholder preferences over the space of possible decisions and the resulting outcomes---such as preferences for balancing spending parity against court appearance rates.
We then optimize over the space of decision policies,
making trade-offs in a way that maximizes the elicited utility.
To do so, we develop an algorithm for efficiently learning these optimal policies 
from data for a large family of expressive utility functions.
In particular, we use a contextual bandit algorithm to explore the space of policies while 
solving a convex optimization problem at each step to estimate the best policy based on the available information.
This consequentialist paradigm facilitates a more holistic approach to equitable decision-making.
\end{abstract}

\section{Introduction}
\label{intro}
Statistical predictions are now used to inform high-stakes decisions in a wide variety of domains.
For example, in banking, loan decisions are based in part on estimated risk of default~\citep{leo2019machine};
in criminal justice, judicial bail decisions are based on estimated risk of recidivism~\citep{cadigan2011implementing,latessa2010creation,goel2018accuracy,milgram2014pretrial};
in healthcare, algorithms identify which individuals receive limited resources, including HIV prevention counseling and kidney replacements~\citep{wilder_2021,friedewald2013kidney,obermeyer2019dissecting};
and in child services, screening decisions are based on the estimated risk of adverse outcomes~\citep{brown2019toward,chouldechova2018case,de2020case,shroff2017predictive}.
In these applications and others, equity is a central concern.
In particular, the machine learning community has proposed numerous methods to constrain predictions to achieve formal statistical properties,
such as parity in decision rates or error rates
across demographic groups~\citep{barocas2017fairness,chouldechova2018frontiers,corbett2023measure,chohlas2023designing}.

To illustrate this traditional approach to designing equitable algorithms, 
consider a government agency that provides free rides for people to get to court~\citep{brough2022transportation}.
Missed court dates can lead to severe penalties, 
including incarceration,
and so improving court appearance rates can reduce social harms~\citep{chohlas2023automated}.
When designing this program,
one might first use historical data to estimate the effect of a ride on increasing each person's likelihood of appearing at court, as well as the cost of providing them with a ride.
Then, in an effort to distribute benefits fairly, 
one might allocate assistance to those with the highest estimated benefit per dollar 
while constraining per-person spending to be equal across demographic groups.
The implicit hope in past literature is that one achieves fairness by
imposing an axiomatic constraint on decisions: spending parity.

Although intuitively reasonable, axiomatic approaches to fairness can cause unexpected harms.
For example, imagine members of one group live farther from the courthouse, making it more costly to provide them rides.
Enforcing equal spending across groups would typically result in fewer rides overall, 
and accordingly lower appearance rates. 
More generally, 
traditional axiomatic approaches to fairness typically do not consider the downstream consequences of constraints, 
and thus fail to engage with the difficult trade-offs at the heart of many policy problems.

We propose an alternative, consequentialist framework to algorithmic fairness. 
In this framework, rather than imposing fairness axioms, one begins by eliciting stakeholder preferences over the space of potential decisions and resulting outcomes.
For example, 
in designing our hypothetical transportation program,
one would assess the degree to which 
stakeholders are willing to trade court appearances for reductions in spending disparities across groups.
Then, using these preferences, we compute a decision-making policy with the largest expected utility while adhering to budget constraints. 
Given historical data on decisions and outcomes, 
we show that optimal decision policies can be efficiently derived for a large and expressive family of utility functions 
by solving a linear program (LP).

We further show how to efficiently learn optimal policies while rolling out new programs in the absence of historical data.
Our approach here is inspired by the success of Thompson Sampling~\citep{chapelle2011empirical} and optimism-under-uncertainty methods~\citep{auer2002finite} in multi-armed bandits. 
In contrast to the standard contextual multi-armed bandit setting, we consider a multifaceted, structured objective to account for complex preferences and budget constraints inherent to many real-world applications. 
As such, our actions at each iteration are guided by solving an LP as above.

The rest of our paper is structured as follows.
In Section~\ref{sec:related} we review the related literature, connecting and contrasting our approach to ideas in fair machine learning, fair division, multi-objective optimization, and reinforcement learning.
In Section~\ref{sec:selecting-policies} we illustrate the trade-offs inherent to many policy problems---and the concomitant benefits of a consequentialist perspective over an axiomatic approach.
To do so, we draw on client data from the Santa Clara 
County Public Defender Office to consider the costs and benefits of a hypothetical transportation assistance program.
We also describe the results of a survey that 
gauged stakeholders' willingness to sacrifice court appearances to reduce spending disparities across race groups.
Given such preferences, as well as historical data on outcomes, in Section~\ref{sec:problem} we formally state and solve the corresponding policy optimization problem.
In Section~\ref{sec:bounds}, we theoretically derive sample complexity bounds on learning optimal policies in the absence of historical data.
Finally, in Section~\ref{sec:online_learning},
we introduce and evaluate an adaptive approach to learning optimal policies, 
combining contextual bandits with the optimization solution described in Section~\ref{sec:problem}.
We end with some concluding thoughts in Section~\ref{sec:discussion}.

\section{Related Work}
\label{sec:related}

Our work draws on research in algorithmic fairness, fair division, multi-objective optimization, and contextual bandits with budgets---connections that we briefly discuss below.

Over the last several years, there has been increased attention on designing equitable machine learning systems~\citep{buolamwini2018gender,raji2019actionable,blodgett17,caliskan,de2019bias,ali2019discrimination,datta2018discrimination,obermeyer2019dissecting,goodman2018machine,chouldechova2018case,koenecke2020racial,shroff2017predictive,chohlas2023designing}, 
and associated development of formal criteria to characterize fairness~\citep{barocas2017fairness,chouldechova2018frontiers,corbett2023measure,gupta2020too}.
Some of the most popular definitions demand parity in predictions across salient demographic groups,
including parity in mean predictions~\citep{feldman2015certifying} or error rates~\citep{hardt2016equality}.
Another class of fairness definitions aims to blind algorithms to protected characteristics, including through their proxies~\citep{kilbertus2017avoiding,wang2019equal,coston2020counterfactual,kusner2017counterfactual,nabi2018fair,zhang2018fairness,chiappa2018causal,wu2019pc,nyarkobreaking,nilforoshan2022}.

All the above approaches conceptualize the equity of algorithmic decisions in terms of universal rules 
(e.g., error rate parity)
rather than considering the consequences of decisions.
Recent work has noted limitations to this axiomatic approach, 
which has otherwise dominated the fair machine learning literature~\citep{cowgill2019economics,cowgill2020algorithmic,corbett2017algorithmic,kasy_2021,grgichlaca2022}.
Some recent exceptions have begun to consider algorithmic decision-making from a consequentialist perspective~\citep{liu2018delayed,viviano2023fair,fang2022fairness,donahue2020fairness,coston2020counterfactual,nilforoshan2022,card2020consequentialism,barabas2018interventions}.
For example, \citet{nilforoshan2022} show that common causal definitions of algorithmic fairness lead to Pareto-dominated policies.
However, although these papers adopt a consequentialist approach to varying degrees, they do not consider the problem of efficiently learning optimal policies, as we do here.

In a related thread of research on fair division problems, groups of individuals decide how to split a limited set of resources among themselves~\citep{bertsimas2011,gal2017fairest,caragiannis2012efficiency,brams1996fair}.
The broad aim of that work---to equitably allocate a limited resource---is similar to our own, but it differs in three important respects.
First, canonical fair division problems seek to arbitrate between individuals with competing preferences (e.g., as in cake-cutting style problems~\citep{procaccia2013cake}), rather than adopting the preferences of a social planner, as we do.
Second, and relatedly, much of the fair division literature, like the algorithmic fairness literature, takes an axiomatic approach to fairness,
identifying allocations that have properties posited to be desirable, such as envy-freeness~\citep{cohler2011optimal}. 
Although that perspective is useful in many applications, it does not explicitly consider the preferences of policymakers, which may be incompatible with these axiomatic constraints.
Finally, work on fair division problems typically does not try to learn causal effects of allocations on downstream outcomes from data, such as the heterogeneous effect of transportation assistance on appearance rates.

In many real-world settings, decision makers have competing priorities, 
linking our work to the large literature on learning to optimize in multi-objective environments~\citep{zuluaga2013active}.
Such inherent trade-offs have recently been considered in the fair machine learning community (e.g.,~\citet{corbett2017algorithmic,cai2020fair,rolf2020balancing}); however, there has been little work on creating equitable learning systems that account for competing objectives. 
Relatedly, a large and growing body of work has shown that one can often efficiently elicit preferences for complex objectives, even 
in high-dimensional outcome spaces~\citep{linpreference,furnkranz2010preference, chu2005preference}.

One particularly challenging aspect of our setting is handling budget constraints (e.g., we may only be able to provide transportation assistance to a limited number of clients). 
Recent work has proposed methods for learning 
decision policies with fairness or safety constraints through
reinforcement learning~\citep{thomas2019preventing} and contextual bandit algorithms~\citep{metevier2019offline}, given access to a batch of prior data. 
That work, however, neither addresses learning with budget constraints nor handles the exploration-exploitation trade-off required for online learning. 
A related study~\citep{patil2021achieving} on online multi-armed bandits considered minimizing regret while ensuring that each arm is played a minimal number of times,
but did not consider context-specific decision policies and fairness in resource allocations or budget constraints,
as we do here.
Budget constraints have been considered in a more general form of knapsack constraints in bandit settings. 
\citet[Ch. 10]{slivkins2019introduction} provides a review of such work, focusing on the primary literature, which has considered the (non-contextual) multi-armed bandit setting. 
Earlier work on  contextual multi-armed bandits with knapsacks~\citep{badanidiyuru2014resourceful, agrawal2016efficient} provided regret bounds but lacked computationally efficient implementations. \citet{agrawal2016near} later proved  regret guarantees for  linear contextual bandit with  knapsacks. 
\citet{wu2015algorithms} provide a computationally tractable, approximate linear programming method for online learning for contextual bandits with budget constraints. They do not consider multi-objective optimization, and their analysis and experiments do not address continuous or large state spaces, which make their work less applicable for equitable decision making in many settings of interest.

\section{Selecting Policies in the Presence of Trade-Offs}
\label{sec:selecting-policies}

We begin, in Section~\ref{sec:motivation}, by describing our motivating example of providing transportation to individuals with mandatory court dates. Using client data from the Santa Clara County Public Defender Office, we show that allocating benefits to maximize appearance rates induces spending disparities across race groups.
Then, in Section~\ref{sec:tradeoffs}, we continue by explicitly illustrating the inherent tension between maximizing appearance rates and equalizing spending---and arguing that popular axiomatic approaches to fairness can lead to unintended harms.
Finally, in Section~\ref{sec:survey}, we describe the results of a survey aimed at eliciting people's willingness to trade court appearances for lower spending disparities.

\subsection{Motivating Example}
\label{sec:motivation}

Consider the problem of allocating rideshare assistance to individuals who are required to attend mandatory court dates.
The consequences of missing a court date can be severe.
Often, after an individual misses a court appearance, judges will issue a ``bench warrant", which can lead to the individual's arrest at their next contact with law enforcement, 
and possibly weeks or months of jail time~\citep{fishbane2020behavioral,chohlas2023automated}.
Despite these consequences, some individuals struggle to attend court because of significant transportation barriers~\citep{mahoney2001pretrial,brough2022transportation,allen2022fta}. 
Government agencies---including public defender offices---may therefore aim to improve appearance rates by offering transportation assistance to and from court for a subset of these individuals with the greatest transportation needs.
This type of intervention has promise for improving appearance rates by alleviating transportation burdens many clients face, as has been demonstrated in medical settings~\citep{chaiyachati2018rideshare,vais2020rides,rand2021,lyft2020}.
As we discuss in Section~\ref{sec:discussion}, it is important to note that there are many alternative policy approaches to this issue, including discouraging judicial use of incarceration after an individual misses court.

A natural algorithmic approach for allocating rides is to prioritize those with the largest estimated treatment effect per dollar.
In particular, suppose we have access to a rich set of covariates, $X_i$, for each individual $i$, such as their age, alleged offense, and history of appearance. 
Based on these covariates, we could then estimate each individual's likelihood of appearance in the absence of assistance, $\hat{Y}_i(0)$, and their likelihood of appearance if provided with a ride, $\hat{Y}_i(1)$.
These probabilities might, for example, be estimated using historical data on past outcomes, or a randomized experiment.
Finally, we could sort individuals by $\rho_i = [ \hat{Y}_i(1) - \hat{Y}_i(0)] / c_i$, where $c_i$ is the cost of providing a ride to the $i$-th individual, and offer assistance to those with the highest values of $\rho_i$ until the budget is exhausted.

This strategy aims to achieve the highest appearance rate given the available budget. 
However, in so doing, it implicitly prioritizes those who live closest to the courthouse---for whom rides are typically less expensive---which could lead to unintended consequences.
For example, consider the Santa Clara County Public Defender Office (SCCPDO) in California, which represents tens of thousands of indigent clients every year.
Like many American jurisdictions, Santa Clara County, which includes San Jose, is geographically segregated by race (Figure \ref{fig:map}).
In particular, Santa Clara's Vietnamese population, one of the county's largest ethnic minorities, 
does not tend to live as close to the courthouse as other racial or ethnic groups,
including white individuals.

\begin{figure}[t]
	\centering
	\begin{subfigure}[t]{0.48\columnwidth}
		\centering
		\includegraphics[width=\columnwidth]{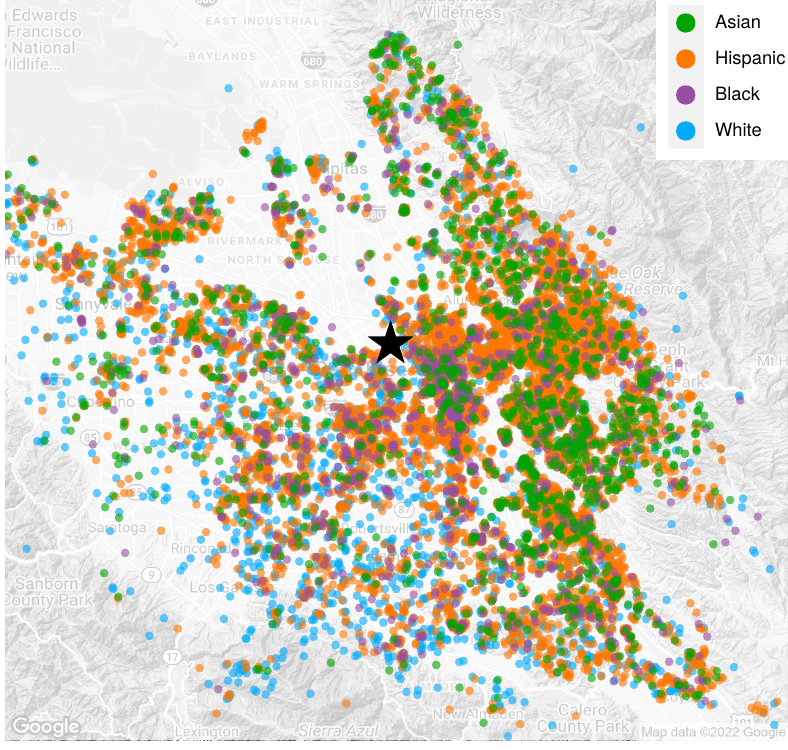}
		\caption{Santa Clara client locations. Each dot has been randomly perturbed to preserve privacy.}
		\label{fig:map}
	\end{subfigure}
	\quad
	\begin{subfigure}[t]{0.47\columnwidth}
		\centering
		\includegraphics[width=\columnwidth]{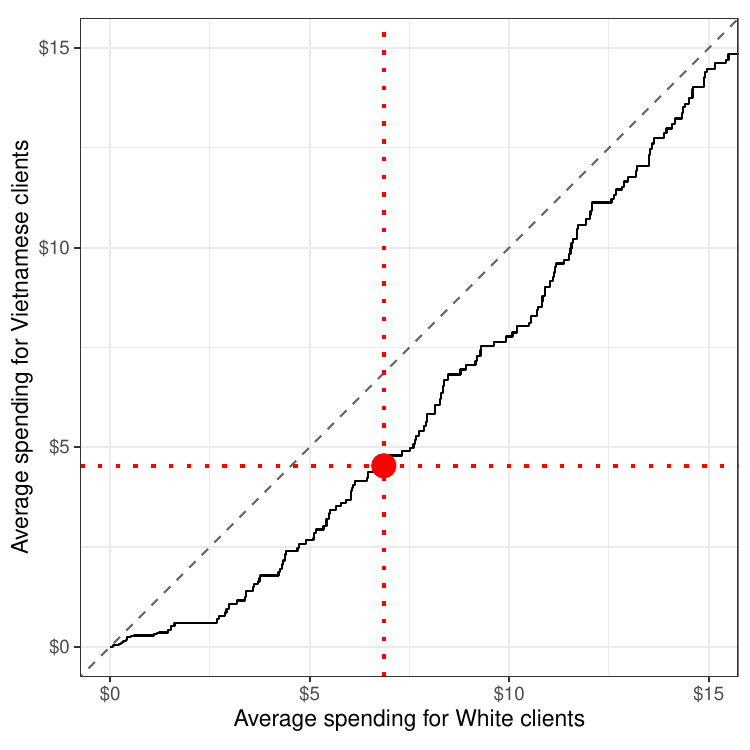}
		\caption{Average per-person spending for Vietnamese and white clients in the absence of parity constraints.}
		\label{fig:spending_curve}
	\end{subfigure}
	\vskip 0.1in
	\caption{
	The map in \subref{fig:map} shows the geographic distribution of the client base of the Santa Clara County Public Defender Office. 
    The star on the map marks the location of the main county courthouse, where most clients are required to appear for court appointments. 
    The plot in \subref{fig:spending_curve} explores the consequence of following a policy that provides rides to those with the highest estimated treatment effect per dollar without parity constraints.
    This policy 
    would result in higher average per-person spending for white individuals than for Vietnamese individuals.
    The red point shows that a hypothetical annual ride budget of \$50,000 would result in an average per-person 
    spending amount of \$6.86 for white individuals and an average per-person spending amount of \$4.54 for Vietnamese individuals.
    }
	\label{fig:map_and_spending_curve}
\end{figure}

To understand the impacts of a strategy that optimizes exclusively for appearance, we start with a dataset of 65,193 court dates handled by SCCPDO between January 1, 2017 and December 4, 2023. 
For the sake of consistency,
this population of court dates 
consists solely of clients' first court date after arraignment.
For clients with court dates after January 1, 2021, 
we use the historical data from 2017--2020 
to model $Y_i(0)$ with a logistic regression model 
based on age, race/ethnicity, offense severity (misdemeanor or felony), two-year appearance history, the day of the week and month of the court appearance, and the distance from the client's home to the courthouse.
For simplicity, we assume $Y_i(1) = 1$, meaning that all individuals who receive a ride attend court.
Finally, we assume rides cost \$5 per mile in each direction, 
in line with current rideshare prices.

Under the naive optimization approach outlined above, Figure~\ref{fig:spending_curve} shows per-capita spending for white and Vietnamese clients across different overall transportation budgets.
For example, given an annual budget of \$50,000,
a policy that allocates rides to those with the highest estimated treatment effect per dollar would end up spending, on average, \$6.86 for every white client, but only \$4.54 on average per Vietnamese client.
Policymakers and other stakeholders may deem this disparity to be undesirable, and may thus be willing to accept lower overall appearance rates 
in return for more equal spending across groups.

\subsection{Exploring Inherent Trade-offs}
\label{sec:tradeoffs}

To further explore the tradeoff between appearance rates and spending parity,
we now consider a synthetic client population with
5,000 Black and 5,000 white clients.
For simplicity, we assume that
each client has a 75\% chance of appearing at court in the absence of rideshare assistance, 
and is guaranteed to appear if provided a ride.
Further, we set a fixed annual budget of \$5 per person, or \$50,000 total.
Finally, we assume that Black clients live farther from court on average.
Consequently, the average expected treatment effect per dollar is lower for Black clients than for white clients.
This pattern induces a tension between maximizing total appearances and equalizing spending across the two groups.\footnote{Optimizing for parity across protected demographic groups, including race groups, is legally impermissible in some contexts in the U.S., as we discuss more in Section~\ref{sec:discussion}.}
We describe the data-generating process for this synthetic population in detail in Appendix~\ref{appx:trade_offs}.

\begin{figure}[t]
\begin{center}
\centerline{\includegraphics[width=0.65\columnwidth]{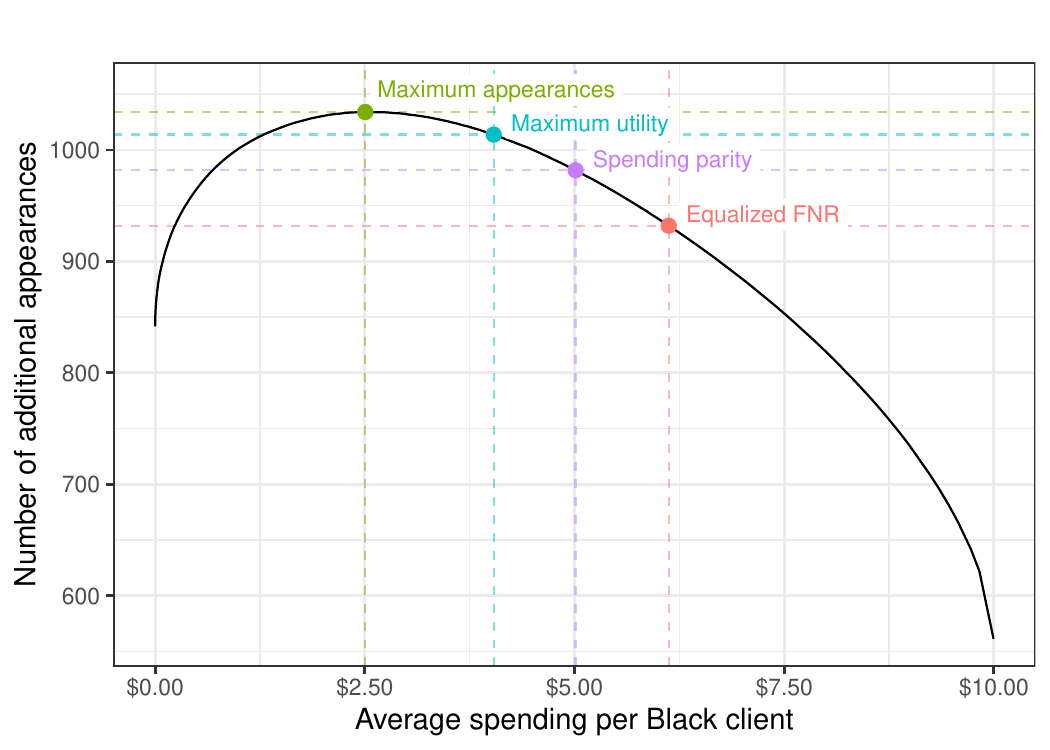}}
\caption{
The Pareto frontier for a stylized population model, 
showing the trade-off between appearances and 
spending per Black client.
The vertical axis shows expected additional appearances relative to a policy that does not provide rideshare assistance to any clients.
Under this model, 
common heuristics (e.g. maximizing appearances, and demanding demographic or error-rate parity) lead to sub-optimal policies.
}
\label{fig:pareto-frontier}
\end{center}
\end{figure}

In Figure~\ref{fig:pareto-frontier}, 
we trace out the Pareto frontier for this example, which shows how the maximum possible number of appearances (on the vertical axis) 
varies under different allocations of rideshare assistance to Black clients (on the horizontal axis).
Each point on the frontier corresponds to a threshold policy that provides assistance to clients with the largest treatment effects in each group, subject to demographic and budget constraints.

Along the Pareto frontier, a policymaker ostensibly has more and less preferred outcomes.
For example,
imagine that a given policymaker's utility is maximized at the blue point on the curve.  
In contrast, the point at the crest of the curve (in green) achieves the highest number of overall appearances, 
but is a suboptimal policy because it underspends on Black clients,
at least according to the stakeholder's preferences.
Similarly,
a policy that achieves perfect spending parity (i.e., the purple point) also yields suboptimal outcomes relative to the policymaker's preferences, 
because too many appearances are lost in order to achieve spending parity.
We also plot the point on the curve
corresponding to 
equal false negative rates (FNR) between groups (in pink).\footnote{In this case, equal FNR means that $\Pr(\pi = 0 \mid Y(0) = 0, Y(1) = 1, G = g) = \Pr(\pi = 0 \mid Y(0) = 0, Y(1) = 1)$.
That is, among those who would benefit from the assistance, an equal proportion do not receive it in both groups.
}
A constraint that demands error-rate parity---as opposed to maximizing utility more directly---can again result in a sub-optimal balance between maximizing appearances and evenly distributing transportation assistance, 
relative to the underlying preferences of the policymaker.
In contrast to the axiomatic approach common to past work,
this simple example helps illustrate the value of viewing decisions from a consequentialist perspective.

\subsection{Eliciting Preferences}
\label{sec:survey}

\begin{figure}[t]
	\centering
	\begin{subfigure}[t]{\columnwidth}
		\centering
		\includegraphics[width=0.5\columnwidth]{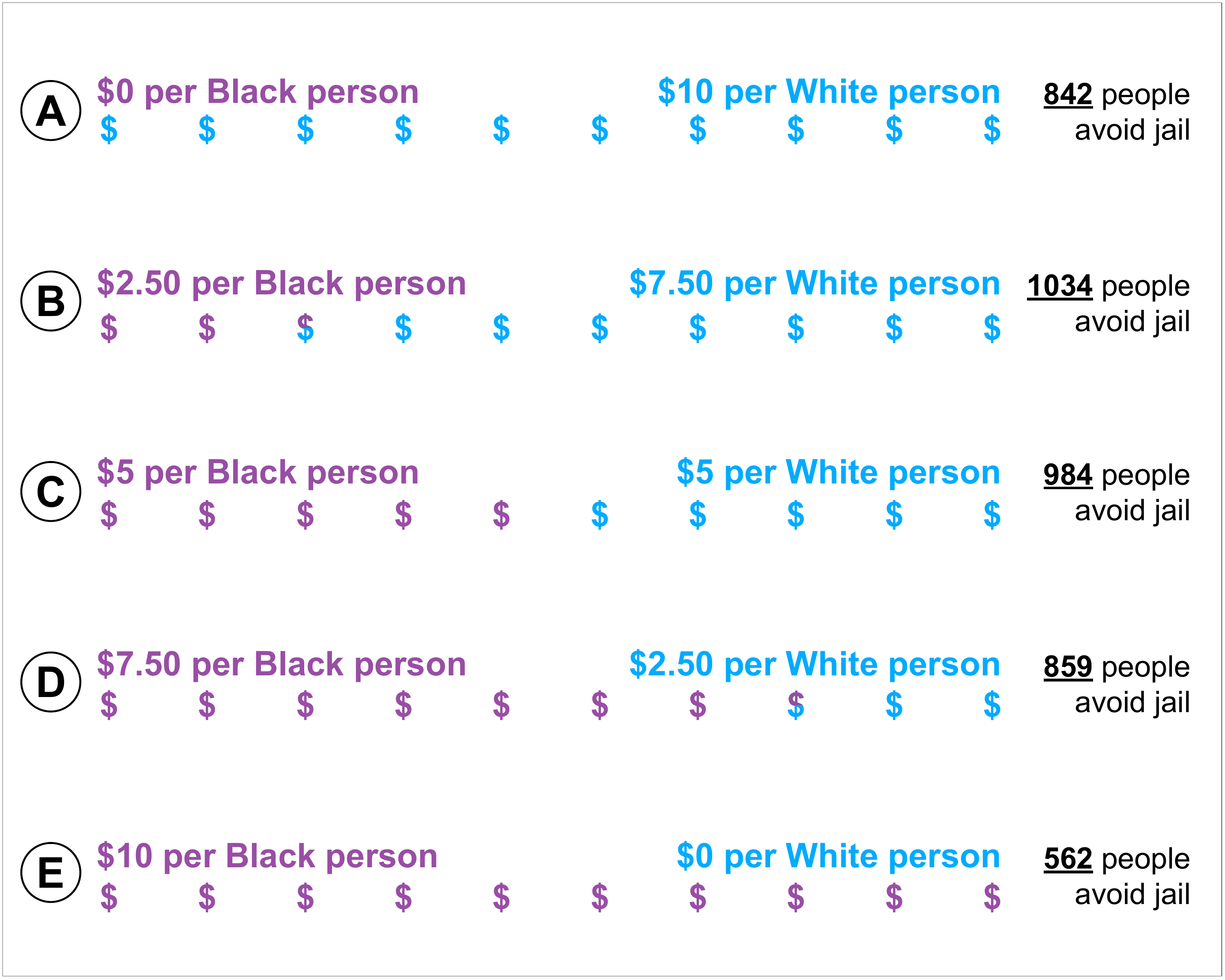}
		\caption{Graphic shown to survey participants.
}
        \vspace{1em}
		\label{fig:survey-graphic}
	\end{subfigure}
	\begin{subfigure}[t]{\columnwidth}
		\centering
		\hspace{-1.5em}\includegraphics[width=0.72\columnwidth]{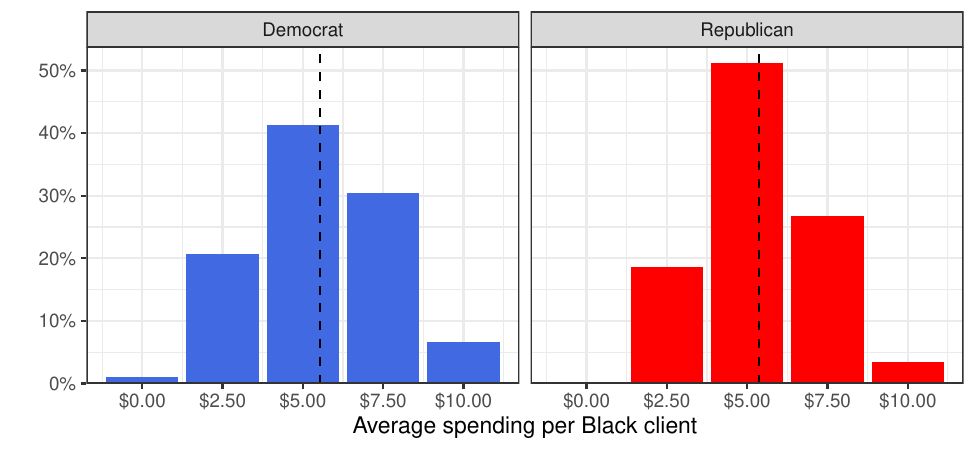}
		\caption{Survey results from 297 respondents, split by self-identified U.S. political party affiliation. Mean preferred allocations are represented by the vertical dashed line.}
		\label{fig:survey_results}
	\end{subfigure}
	\vskip 0.1in
	\caption{The graphic in \subref{fig:survey-graphic} was shown to survey participants to help them select their preferred ride allocation policy. 
    In this hypothetical scenario, option B maximizes appearances, while option C corresponds to spending parity.
    The survey results in \subref{fig:survey_results} show that both Democrats and Republicans prefer policies that spend roughly equal amounts on Black and white clients, 
    but there is a wide range of preferences among members of both groups.
    }
	\label{fig:survey_graphic_and_results}
\end{figure}

We now empirically examine preferences for allocating transportation assistance in our hypothetical scenario above.
To do so, we designed and administered a poll to a diverse sample of 297 Americans.
We ran our survey on the Prolific platform,
selecting the platform's ``U.S. representative sample'' option to recruit respondents,
where respondents' self-identified sex, age, and ethnicity 
is comparable to a random population of U.S. adults, as determined by the U.S. Census.
Survey respondents learned about
our running example of providing clients with free rides to court,
and then read a short description of the hypothetical jurisdiction described above. 
(This prompt is included in full in Appendix~\ref{appendix:survey_prompt}.) We then asked respondents to select their preferred tradeoff among five possible options drawn from the Pareto frontier in Figure~\ref{fig:pareto-frontier}.
To aid in their decision, participants were shown the graphic depicted in Figure \ref{fig:survey-graphic}.
Participants were randomly shown either an ascending or descending version of this graphic to mitigate anchoring to the first options shown.

Our survey results are presented in Figure~\ref{fig:survey_results},
and illustrate two key points.
First, the vast majority of respondents prefer trading at least some ``efficiency'' (i.e., as measured by total number of people who avoid jail due to receiving transportation assistance) in order to spend more money on Black clients.    
This broad preference for incorporating equity considerations into algorithmic decision making mirrors past results~\citep{koenecke2023}.
Second, there is substantial heterogeneity in preferences that elides traditional group boundaries. 
For example, there is considerable variation in preferences within self-identified Democrats and Republicans; at the same time, the average preference is comparable across these two groups.
We observe similar patterns across a number of other demographic characteristics of the respondents, 
including gender and race/ethnicity,
as we show in Appendix~\ref{appendix:survey_prompt}. 
These results suggest that traditional
axiomatic approaches to algorithmic fairness---which do not consider the specific context of decisions---risk yielding policies that do not reflect the preferences of stakeholders.
In contrast, a more consequentialist perspective allows us to develop algorithms that better balance the difficult trade-offs inherent to many policy problems.

\section{Computing Equitable Policies}
\label{sec:problem}

For the rideshare example in the previous section, it is computationally straightforward to trace out the Pareto frontier: for any fixed budget allocated to each group, one can maximize appearances by offering rides to those clients with the largest (estimated) gain in appearance rate per dollar, while constraining spending to the allotted per-group budget.  (We formally show the optimality of this strategy in Appendix~\ref{sec:thresholds}.)
As a result, given preferences over various outcomes (e.g., trading off appearances with spending parity), one can efficiently determine the utility-maximizing allocation strategy.
However, in more complicated scenarios---with more complex preferences and potential actions---it is not immediately clear how to find optimal allocation strategies, even when preferences and treatment effects are fully known.
Fortunately, for a large class of preferences, it is indeed feasible to efficiently compute utility-maximizing policies, as we now describe.
In Section~\ref{sec:online_learning}, we consider the problem of learning optimal polices when preferences are known but treatment effects are not.

To generalize from our running example, consider a sequential decision-making setting where, at each time step, one first observes a vector of covariates $X_i$ drawn from a distribution $\mathcal{D}_X$ supported on a finite state space $\mathcal{X}$,
and then must select one of $K$ actions from the set $\mathcal{A} = \{a_1, \dots, a_K\}$.
For example, in our motivating application, $X_i$ 
might encode an individual's demographics, history of appearance, alleged charges, and distance from court,
and the set of actions might specify whether or not rideshare assistance is offered (in which case, $K=2$).
In general, we allow randomized decision policies $\pi$, 
where the action $\pi(x)$ is (independently) drawn from a specified distribution on $\mathcal{A}$.

In practice, there are often constraints on the distribution of actions taken. 
For example, budget limitations might mean that only a certain amount of money can be spent on average per client, with varying known costs per context and action $c(x, a_k)$.
As such, given a cap $b$ for average per-person expenditures, we require our decision policy $\pi$ to satisfy 

\begin{align*}
 \EE_{X} [c(X, \pi(X))] & = \sum_{x,k} 
\Pr(X = x) \cdot \Pr(\pi(x) = a_k) \cdot c(x, a_k)\\
&\leq b.
\end{align*}
In many common scenarios, we might imagine a setup where one ``control'' action $a_0$ has no cost (i.e., $c(x, a_0) = 0$), while all other available actions are costly
(i.e., $c(x, a_{k}) > 0$ for $k>0$).

To arbitrate between feasible policies (i.e., those that adhere to the budget constraint), policymakers might consider both the direct outcomes of a policy (e.g., on appearance rates) and the relative allocation of benefits across demographic groups.
To formalize this idea, we suppose each action is associated with a potential outcome $Y_i(a_k)$,
and, in particular, taking action $\pi(X_i)$ results in the (random) outcome $Y_i(\pi(X_i))$.
For example, $Y_i(1)$ may indicate whether the $i$-th individual would attend their court date if offered rideshare assistance, and $Y_i(0)$ may indicate the outcome if assistance were not provided.

Now, to facilitate computation, we assume a policymaker's utility $U(\pi)$
of any decision policy $\pi$ can be approximated by a flexible function of the following form:
\begin{equation}
\label{eq:utility}
\begin{aligned}
    U(\pi) &= \EE_{X, Y}[r(X, \pi(X), Y(\pi(X)))] \\
    & \hspace{-7mm} - \sum_{\ell = 1}^L \sum_{g \in \mathcal{G}} \lambda_{g, \ell} 
             \Big\lvert
             \EE_{X,Y}[f_{\ell}(X, \pi(X), Y(\pi(X))) \mid g \in s(X)] - \EE_{X,Y}[f_{\ell}(X, \pi(X), Y(\pi(X)))] 
             \Big\rvert,
\end{aligned}
\end{equation}
where 
$r$ and $f_{\ell}$ are fixed functions that parameterize this class of utilities,
$\lvert \cdot \rvert$ is an absolute value, 
$\lambda_{g,\ell}$ are non-negative constant parameters,
and
$s(X_i) \subseteq \mathcal{G}$ is a set of associated identities for each individual, where $\mathcal{G}$ is a finite set.
In discussions of algorithmic fairness, special attention is often paid to these groups, which may consist of legally protected characteristics.
For example, $s(X_i$) might specify both an individual's race and gender.

The first term in $U(\pi)$ captures the social value 
directly associated with each decision, and
the second term penalizes differences in allocations and outcomes across groups.
For example, in our motivating application,
we might set 
\begin{equation}
\label{eq:r}
r(x, a, y) = (a + c_1 y) \cdot \left (1 + c_2 \cdot \mathbb{I}_{\textup{frequent}}(x) \right),
\end{equation}
where 
$a \in \{0,1\}$ indicates whether rideshare assistance is provided, 
$y \in \{0,1\}$ indicates whether a client appeared at their court date,
$\mathbb{I}_{\textup{frequent}}(x)$ indicates whether an individual is in frequent contact with law enforcement,
and the positive constants $c_1$ and $c_2$ characterize the relative values of the terms.
(In Eq.~\eqref{eq:r}, we do not multiply $a$ by a constant, since the overall scale of $r$ is arbitrary.)
This choice of $r$ encodes the (hypothetical) policymaker's belief
that: 
(1) appearing at one's court date is better than not appearing;
(2) receiving  rideshare assistance is better than not receiving it, regardless of the outcome;
and (3) the value of both assistance and appearance 
is greater for those who frequently encounter law enforcement (i.e., those for whom an open bench warrant is more likely to result in jail time because they are more likely to encounter law enforcement).

In addition to preferring
transportation assistance policies that boost appearance rates, 
a policymaker might also prefer those for which we spend similar amounts per person across demographic groups,
to ensure such investments are broadly applied across an agency's jurisdiction.
The second term of $U(\pi)$ can be used to encode these parity preferences.
For example, setting $f(x, a, y) = c(x, a)$ would encode a preference for spending parity.
Depending on the application, one could imagine similarly penalizing a given policy if the distribution of \textit{actions} or \textit{successes} were unequal across groups.

In practice, to encode preferences in this way, one might first show stakeholders anticipated outcomes of various hypothetical policies, akin to our survey above.
We could then sweep over parameters to produce a utility function of the appropriate form that accurately captures the elicited preferences. 
Importantly, and in contrast with an axiomatic approach, our consequentialist paradigm is predicated on the belief that there are not universal, context-independent constraints on policies.
Rather, the utility of a policy depends critically on how much one objective must be sacrificed to achieve another. 

Given this setup, our goal is to find a policy $\pi^*$ that maximizes utility while staying within budget. Formally, we seek to solve the following optimization problem:
\begin{equation}
\begin{aligned}
\label{eq:opt-problem}
    & \pi^* \in \arg \max_{\pi} \ U(\pi)\\
    & \text{subject to:} \ \EE_{X} [c(X, \pi(X))] \leq b.
\end{aligned}
\end{equation}
We next discuss how to efficiently solve this optimization problem.

\subsection{Computing Optimal Decision Policies}
\label{sec:LP}

To compute optimal policies, 
we assume, in this section, that one knows the distribution of $X$ and the conditional distribution of the 
potential outcomes $Y(a_k)$ given $X$---i.e., $\DD(X)$ and $\DD(Y(a_k) \mid X)$.
(In Section~\ref{sec:online_learning}, we consider how to learn optimal policies when historical data on treatment effects are not known.)
Given this information, we show the optimization problem in Eq.~\eqref{eq:opt-problem}
can be expressed as a linear program (LP), yielding an efficient method for computing an optimal decision policy.

To construct the LP, first observe that 
any policy $\pi$ corresponds to a matrix 
$v \in \RR_+^{\mathcal{X}} \times \RR_+^K$,
where $v_{x,k}$ denotes the probability $x$ is assigned to action $k$.
Thus, the complete space of policies $\Pi$ can be written as:
\begin{equation*}
\Pi = \left \{v \in \RR_+^{\mathcal{X}} \times \RR_+^K \, \left | \, \forall x \in \mathcal{X}, \ \sum_{k=1}^K \right . v_{x,k} = 1 \right \},
\end{equation*}
and we can accordingly view the components $v_{x,k}$ of $v$ as decision variables in our LP.
Now, in this representation, the budget constraint $\EE_{X} [c(X, \pi(X))] \leq b$ in Eq.~\eqref{eq:opt-problem} can be expressed as
a linear inequality on the decision variables:
\begin{equation*}
    \sum_{x, k} \Pr(X = x) \cdot v_{x,k} \cdot c(x, a_k) \leq b.
\end{equation*}
Finally, we need to express the utility $U(x)$ in linear form. First, note that:

\begin{align*}
  U(\pi) &= \sum_{x, k} \EE_Y [r(x, a_k, Y(a_k)) \mid X = x] \cdot \Pr(X = x) \cdot v_{x,k} \\
  & \hspace{10mm} - \sum_{\ell} \sum_g \lambda_{g,\ell} \Bigg | \sum_{x,k}   \Bigg(  \frac{ \mathbb{I}(g \in s(x)) \Pr(X = x)}{\Pr(g \in s(X))} \cdot \EE_Y [f_{\ell}(x, a_k, Y(a_k)) \mid X = x ]\\
  & \hspace{60mm}-\Pr(X = x) \cdot \EE_Y [f_{\ell}(x, a_k, Y(a_k)) \mid X = x ] \Bigg) v_{x,k}  \Bigg |.
\end{align*}
Due to the absolute value, the expression above is not linear in the decision variables. But we can use a standard construction to transform it into an expression that is. In general, 
suppose we aim to maximize an objective function of the form 
\begin{equation}
\label{eq:abs-opt}
    \alpha^T v - \sum_{g,\ell} \lambda_{g,\ell}|\beta_{g,\ell}^Tv|,
\end{equation}
where $\alpha$ and $\beta$ are constant vectors.
We can rewrite this optimization problem as a linear program that includes additional (slack) variables $w_{g,\ell}$:
\begin{equation}
\label{eq:abs-lp}
\begin{aligned}
\text{Maximize:} & \quad \alpha^T v - \sum_{g,\ell} \lambda_{g,\ell} w_{g,\ell} \\
\text{Subject to:} & \quad 0 \leq w_{g,\ell}, \\
& \quad -w_{g,\ell} \leq \beta_{g,\ell}^T v \leq w_{g,\ell}.
\end{aligned}
\end{equation}
For completeness, we include a proof of this equivalence in Appendix~\ref{appendix:LP}.

Putting together the pieces above, we now write our policy optimization problem in Eq.~\eqref{eq:opt-problem} as
the  following linear program:
\begin{align*}
& \textup{Maximize:} \\
& \quad \sum_{x,k} \EE_Y [r(x, a_k, Y(a_k)) \mid X = x] \cdot \Pr(X = x) \cdot v_{x,k} - \sum_{g,\ell} \lambda_{g,\ell} w_{g,\ell} \\
& \textup{Subject to:} \\
& \quad v_{x,k},\, w_{g,\ell}  \geq  0 \hspace{4.5mm} \forall \, x,k,g,\ell\\
& \quad \sum_k v_{x,k} = 1 \hspace{4.5mm} \forall \, x,\\
& \quad \sum_{x, k} \Pr(X = x) \cdot v_{x,k} \cdot c(x, a_k) \leq b,~\text{and} \\
& \quad -w_{g,\ell} \leq \sum_{x,k} \Bigg(  \frac{ \mathbb{I}(g \in s(x)) \Pr(X = x)}{\Pr(g \in s(X))} \cdot \EE_Y [f_{\ell}(x, a_k, Y(a_k)) \mid X = x])\\
& \hspace{4cm} - \Pr(X = x) \cdot \EE_Y [f_{\ell}(x, a_k, Y(a_k)) \mid X = x] \Bigg) v_{x,k} \leq w_{g,\ell} \hspace{4.5mm} \forall g,\ell.
\end{align*}

Our approach above is a computationally efficient method for finding optimal decision polices. 
In theory, linear programming is (weakly) polynomial in the size of the input: 
$O(|\mathcal{X}|K + |\mathcal{G}|L)$
variables and constraints in our case.
In practice, using open-source software running on conventional hardware,
we find it takes approximately 1--2 seconds to solve random instances of the problem on a state space of size 
$|\mathcal{X}| = 1,000$
with $|\mathcal{G}| = 10$ groups, 
$K=5$ treatment arms, and $L=1$ parity penalities.\footnote{We used the Glop linear optimization solver, as implemented in Google OR-Tools (\url{https://developers.google.com/optimization/}).
}

\section{Sample Complexity Bounds on Learning Optimal Policies}
\label{sec:bounds}

To solve our policy optimization problem, we have thus far assumed
perfect knowledge of the distribution of potential outcomes $\mathcal{D}(Y(a_k) \mid X)$,
which allows us to compute the necessary inputs for our linear program.
In reality, however, this distribution must typically be learned from observed data. 
One common approach for estimating the impact of actions is to run an experiment in which actions are randomly allocated, 
potentially in a way to ensure that all actions are taken an equal number of times, or to ensure each group of interest experiences all actions evenly.  
Note that such data collection strategies do not adapt in response to observed outcomes of actions (such as some actions yielding higher appearance rates partway through data collection). 
In Section~\ref{sec:sample_bounds}, we formally analyze these non-outcome-adaptive data collection strategies, and provide an upper bound on the number of samples necessary to ensure we can compute a near-optimal allocation strategy for our desired objective. 
In Section~\ref{sec:sample_bounds_costs}, we discuss some considerations relating to experimental cost. 
We present this initial analysis to highlight how in some cases, the amount of data needed may not differ substantially from simpler objectives that do not involve parity constraints. 
The other benefit of this first analysis is that it involves creating an experimental design for data collection in advance, which makes the resulting data easily suitable for standard statistical inference. 
However, in practice there can be significant benefits to changing the data gathering strategy over the course of an experiment, 
as more effective actions can be prioritized faster. 
In Section~\ref{sec:online_learning} we demonstrate this through an alternative, contextual-bandit-based data collection strategy 
that can often learn optimal policies more efficiently  by judiciously exploring the effects of actions. 
We demonstrate the advantages of this alternative strategy in an empirically grounded simulation study.

\subsection{Sample Complexity Bounds }
\label{sec:sample_bounds}
A natural concern for practitioners is whether balancing multiple complicated objectives---like the competing outcomes highlighted in our utility function in Eq.~\eqref{eq:utility}---requires obtaining substantially more data than in traditional, single-objective settings. 
Further, in most domains of practical interest, individuals are described by a set of features, and it is beneficial to know how choices about representing these individuals
impact the amount of data required. 
To address these considerations,
we provide upper bounds on the samples 
needed to construct near-optimal policies with high probability, 
focusing on spending parity by setting $f(x,a,y) = c(x,a)$ 
(following our running example).
Our aim in this analysis is not to provide tight sample complexity bounds, but rather to
examine at a high level how additional parity  objectives and modeling choices affect the amount of data required.  
Our results suggest that one may not need
much more data to learn a multi-objective policy that incorporates equity preferences compared with a single-objective reward-maximizing policy, 
and that known structure on the data generating process can substantially reduce the amount of data required.

Our work is related to a deep literature in multi-armed bandits and contextual multi-armed bandits (see~\cite{lattimore2020bandit} for a fairly recent textbook overview). The majority of this research has focused on providing cumulative regret guarantees of online, adaptive algorithms for a wide range of settings, including seminal results for finite armed bandits~\citep{auer2002finite} and linear contextual bandits~\citep{abbasi2011improved}, as well as more recent interest in logistic models (e.g.~\citet{li2017provably,jun2021improved}). Approaches that minimize cumulative regret bounds can be different from algorithms that provide 
sample complexity bounds that are probably approximately correct (PAC)---i.e., 
methods that after a sufficient amount of data, output a decision policy that is near optimal with high probability. 

Interestingly, prior work (e.g.,~\citet{jin2018q}, Section 3.1) has provided an online-to-batch reduction that can be used to convert a contextual multi-armed bandit algorithm with a cumulative regret result to a sample complexity bound on the number of samples needed to extract a near-optimal policy with high probability. 
However, most contextual MAB algorithms with cumulative regret guarantees rely on selecting actions under the principle of optimism under uncertainty with respect to the immediate estimated reward for the current context. 
The resulting regret bound is defined with respect to the best action that could have been selected. In contrast, in our setting the objective is to compute a policy $\pi$ that maximizes the utility function $U(\pi)$ which includes both a reward maximization term and a spending parity term. 
In general, the optimal policy in our setting will not match an optimal policy that maximizes only the reward. This implies we cannot directly leverage an online-to-batch reduction from existing algorithms with cumulative regret bounds, since the regret bounds provided by those algorithms will not provide regret bounds for our setting. To our knowledge, none of the existing online contextual bandit algorithms consider additional parity objectives, or a joint policy across contexts, as in our work.  

There is fairly limited work on MAB and contextual MAB algorithms that directly provide PAC guarantees. \citet{bandit_complexity_lower_bound}'s foundational work provided sample complexity bounds for multi-armed bandits with a finite set of arms, and we will build on their work for providing sample complexity bounds for our setting given a finite set of contexts and arms/actions, also known as the tabular setting. 
Concurrent to the development of this work, there has been some recent interest in sample complexity bounds for contextual bandits (e.g.,~\cite{zanette_doe,pacchiano2023experiment,li2022instance}) which we will discuss further under different assumptions of the underlying data generating process.

We now introduce some additional assumptions. As in Sections~\ref{sec:problem} and \ref{sec:LP}, we further assume throughout this section that the state space $\mathcal{X}$ is finite, and that the costs and the distribution of $X$ are known.
In practice, information on the distribution of $X$ can often be estimated from historical data, before any interventions are attempted.
Let $\pistar$ be an optimal policy solution, as defined in Eq.~\eqref{eq:opt-problem}, with corresponding utility $U(\pistar)$. We define the estimated utility function $\hat{U}(\pi)$ for a particular decision policy as 
\begin{equation}
\begin{aligned}
    \Uhat(\pi) & = \EE_{X, Y}[\hat{r}(X, \pi(X), Y(\pi(X)))] \\
     & \hspace{5mm} - \sum_{g \in \mathcal{G}} \lambda_g 
             \left|
             \EE_{X}[c(X, \pi(X)) \mid g \in s(X)] - \EE_{X}[c(X, \pi(X))] 
             \right|, \label{eqn:uhat}
\end{aligned}
\end{equation}where $\rhat$ is the estimated reward function learned from data. Let $\pihat$ be a solution to the optimization problem in Eq.~\eqref{eq:opt-problem}, where we maximize $\Uhat(\pi)$ instead of $U(\pi)$. 
Further, let $r(x, k) = r(x,a_k,Y_{X=x}(a_k))$ be the (random) reward if action $a_k$ is taken in the context $x$, where $Y_{X=x}(a_k)$ is the (random) potential outcome conditional on the given context. 
Note that the randomness in $r(x, k)$ stems entirely from the randomness in the potential outcomes $Y_{X=x}(a_k)$.

First we present a simple lemma that allows us to bound the utility error by the reward estimation errors which we will use for the proofs of the theorems.

\begin{lemma}\label{lemma:rew_to_val_err_bound}
The loss of utility due to using $\pihat = \argmax_{\pi} \hat{U}(\pi)$ is bounded by
\begin{equation}
            U(\pistar) - U(\pihat)  
            \leq 2 \sum_x p_x \max_k |r_{xk} - \rhat_{x k}| \label{eqn:reward_diff}
 \end{equation}
\end{lemma}

\begin{proof}{Proof.}
Both $\pihat$ and $\pistar$ by definition satisfy any provided constraints. Then 
\begin{eqnarray}\label{eq:utility_triangle_ineq}
    U(\pistar) - U(\pihat) &=& U(\pistar) - \Uhat(\pihat) + \Uhat(\pihat) - U(\pihat)
  \\
  &\leq& U(\pistar) - \Uhat(\pistar) + \Uhat(\pihat) - U(\pihat).
\end{eqnarray}
where the second equation follows because  $\pihat = \argmax_{pi} \Uhat(\pi)$, and so $\Uhat{\pistar} \leq \Uhat{\pihat}$. 

Since the parity part of the utility function depends only on the policy, and not the rewards, it cancels out in Equation~\ref{eq:utility_triangle_ineq}, leaving

\begin{align}
    U(\pistar) - \Uhat(\pistar) + \Uhat(\pihat) - U(\pihat) 
    &= \sum_x p_x \sum_k \pistar_{xk}(r_{xk} - \rhat_{xk}) + \sum_x p_x \sum_k \pihat_{xk}(\rhat_{xk} - r_{xk}) \\
    &\leq 2 \sum_x p_x \max_k |r_{xk} - \rhat_{x k}| 
\end{align}
\qed
\end{proof}

We now present upper bounds on the sample size needed to learn near-optimal policies.
Specifically, for fixed $\epsilon, \delta > 0$,
we provide sample bounds which ensure 
the utility gap $U(\pistar) - U(\pihat)$
is small with high probability, i.e.,
$\P(U(\pistar) - U(\pihat) < \epsilon) > 1- \delta$.
We prove these bounds under three different common distributional assumptions on the reward model: tabular, linear and logistic reward models: 
\begin{enumerate}
    \item (Tabular Rewards) 
We assume $r(x, k) \eqdist f(x, k) + \eta$, where $\eta \sim \sigma^2$-subGaussian and 
    $\eta$ is independent across draws of the reward function.
    \item (Linear Rewards) 
    We assume there are (known) features $\phi(x,a_k) \in \Reals^d$ of the state and action, and (unknown) parameters $\thetastar \in \Reals^d$ such that
$r(x, k) \eqdist \phi(x,a_k)^T\thetastar + \eta$,
    where $\eta \sim \sigma^2$-subGaussian and 
    $\eta$ is independent across draws of the reward function.
\item (Logistic Rewards) 
    We assume there are (known) features $\phi(x,a_k) \in \Reals^d$ of the state and action, and (unknown) parameters $\thetastar \in \Reals^d$ such that
$\P(r(x, k) = 1) = \text{logit}^{-1}(\phi(x,a_k)^T\thetastar)$,
    where the reward is independent across draws.
\end{enumerate}
Full proofs for this section are in Appendix~\ref{appx:proofs_sample_bounds}.

\begin{theorem}[Tabular Rewards]
\label{thm:rct_tabular}
Assume the reward is tabular. Assume $n$ samples are collected in a round-robin fashion (i.e., for each context $x$, select the least-sampled action $a_k$ in that context, breaking ties arbitrarily). Further assume that the data are used, per $(x,a)$ pair, to estimate a maximum likelihood reward model $\hat{r}(x,a)$ that is used to define $\Uhat$ (see Equation~\ref{eqn:uhat}) and $\pihat = \argmax \Uhat$. 
Then for $\epsilon > 0$, $\delta > 0$, $\lambda_g \geq 0$, if 
\begin{equation*}
    n  \geq   16\sigma^2 \frac{|X||A|}{\epsilon^2}\ln\frac{4|X||A|}{\delta} \ln \frac{2|X|}{\delta},
\end{equation*}
then 
$\P(U(\pistar) - U(\pihat) < \epsilon) > 1- \delta$.
\end{theorem}
Standard proofs for tabular multi-armed bandits rely on concentration inequalities on the estimated reward functions~\citep{bandit_complexity_lower_bound}. Unlike this work, we additionally need to estimate the reward per context to ensure the final estimated utility, which is a weighted sum over contexts, is near accurate.  We show it suffices to estimate the reward outcome for a particular $(x,a)$ pair to differing levels of accuracy, based on the probability of the context $x$, which allows our final bounds to be independent of the minimum context probability. This result is identical to finding a policy such that $\sum_x p(x) r(x,pi(x))$ is $\epsilon$-close to optimal. Note that this sample bound is identical whether or not we consider spending parity (i.e., regardless of whether $\lambda_g > 0$ for some $g$ or $\lambda_g = 0$ for all $g$ in Eq.~\eqref{eq:utility}). 
Intuitively, this is the case because
the sample complexity is driven by uncertainty in estimating the rewards. The parity component itself depends only on the allocation across subgroups, which 
can be computed exactly given any policy, independent of the estimated rewards.

Our sample bound in the tabular setting scales linearly with the product of the size of the context space and the action space,
which suggests that prohibitively large sample sizes may be needed in practice.  When the contexts are independent, this dependence is unavoidable, as a sample complexity lower bound shows at least $\frac{|A|}{\epsilon^2}$ samples are required for a single context~\citep{bandit_complexity_lower_bound}. 
Our next theorem proves 
significantly fewer samples are sufficient if the reward function is a linear model.

\begin{theorem}[Linear Rewards]
\label{thm:rct_linear}
Assume the reward is linear
with feature representation $\phi(x,a_k) \in \Reals^d$. 
For any non-adaptive strategy $\pi$ used to collect samples, let 
\begin{align*}
\Sigma(\pi) & = \EE [ \phi(X, \pi(X)) \phi(X, \pi(X))^T] 
= \sum_{x,k} \P(X=x) \cdot \P(\pi(x) = a_k) \cdot \phi(x, a_k) \phi(x, a_k)^T
\end{align*}
be the expected induced covariance matrix.
Also define a problem-dependent constant 
$$\rho_0(\pi) = \max_{x, k} \norm{\Sigma(\pi)^{-1/2}\phi(x, a_k)}/\sqrt{d}.$$ 
There exists a static (it does not update as data is gathered) data collection strategy $\pitilde$ such that, for any $\epsilon > 0, \delta > 0, \lambda_g \geq 0$
and 
\[n \geq \max\{6 \rho_0(\pitilde)^2 d\log(3d/\delta), \ \BigO{\sigma^2 d^2/\epsilon^2}\},\]
with cost incurred \(c \leq n\max_{x k } c(x, a_k)\), we have $\P(U(\pistar) - U(\pihat) < \epsilon) > 1- \delta$.
\end{theorem}
The quantity $\rho_0$ in the above bound is known as `statistical leverage' \citep{hsu2014random}.
If no prior information is available, we know only that
$\rho_0 \leq \norm{\phi}_{2}/ \sqrt{\lambda_{\min}(\Sigma)}$. In the worst case, $\rho_0$ may scale with the condition number of the covariance matrix.
However, in many practical settings  $\rho_0$ 
is not large compared to $1/\epsilon^2$,
and so the upper bound scales like $\sigma^2 d^2 / \epsilon^2$. $\pitilde$ refers  to the data collection strategy in Theorem~\ref{thm:rct_linear}, 
and $\pihat$ refers to the performance of a learned decision policy that maximizes the utility given the gathered data.

The above result was motivated in part by, as we noted earlier,  that in general we cannot directly leverage  cumulative regret bounds for contextual bandits since the bounds relate empirical decisions to the optimal decision for the current context, with no further constraints or objectives. However, concurrent research by~\cite{zanette_doe} on contextual linear bandits, provides a sample complexity result sufficient to  upper bound the expected reward error,
\begin{equation}
\int_{x} p_x \max_{k} | \phi(x,a_k)^T (\theta^* - \hat{\theta})|.
\end{equation}
From our Lemma~\ref{lemma:rew_to_val_err_bound}, we can use this to directly bound our expected utility. Therefore we could also use their data collection strategy and bound and obtain a $O(\frac{d^2}{\epsilon^2})$ sample complexity result, which does not depend on $\rho_0(\pi)$. Their sample complexity results is minimax (in the dominant term, up to constants and log terms) optimal for linear contextual bandits (for both static data collection strategies that do not update based on the observed rewards, and for adaptive ones that do update as rewards are observed). This implies that using their algorithm also yields a minimax (in the dominant term, up to constants and log terms) optimal sample complexity result for our setting, since crucially, the parity objective depends only on the policy.

Given the practical importance of binary rewards, bounds for this setting would also be beneficial. However, while there has been some recent attention to logistic bandits \citep{li2017provably,dong2019performance,jun2021improved}, these papers have focused on cumulative regret guarantees. \cite{jun2021improved} provide some PAC bounds on returning the optimal arm for logistic bandits. We are not aware of sample complexity results for contextual logistic bandits. In  Appendix~\ref{appx:proofs_sample_bounds} we provide some preliminary bounds on the suboptimality of the performance 
of the resource allocation strategy derived from using estimated plug-in parameters for the logistic reward model (Theorem~\ref{thm:rct_logistic}). Our results require strong assumptions, and depend on problem-specific properties and the data collection strategy, suggesting there is significant room for similar results under more relaxed, and constructive, conditions. Contextual multi-armed bandits are an  active research area in the machine learning community, and it is likely our results can benefit from future results on sample complexity algorithms and bounds for contextual bandits.

 \subsection{Cost-Aware Sample Complexity}
\label{sec:sample_bounds_costs}
While we have provided sufficient sample bounds, it is also useful to 
consider
bounds on the experimental cost sufficient to learn a near-optimal policy. Note by this we mean a bound on the cost required to learn a near-optimal policy, not the budget constraint on the learned decision policy. In general the amount of resources available during the experimental period may be different than the resources available during sustained deployment.

In the tabular case, we can prove that an experimental budget of $O(\sum_{x, k} c(x,a_k)  \log(1/\delta)/ \epsilon^2 )$ is sufficient (see Corollary~\ref{cor:cost_upper_bounds_tabular} in the Appendix). When the domain can be modeled with a linear or logistic reward model, we can simply multiply our sample bounds by the maximum cost $\max_{xk} c(x,a_k)$ to get sufficient upper bounds on the experimental cost. In some settings these bounds are likely order optimal in the dominant terms. For example, in the tabular case without parity preferences, when costs are homogeneous across contexts and actions, and the context distribution is uniform, the expected experimental cost must be at least $c|X||A| \log(1/\delta)/\epsilon^2$ 
in the worst case (Theorem~\ref{thm:lower_cost_bounds_uniform_tabular} in the Appendix).

However, in general we expect that there are alternate strategies,
with tighter bounds, that are cost-aware. 
As an illustration, consider a setting with two contexts, two actions, bounded rewards, no parity preferences ($\lambda_g = 0$), 
and a budget $b$ (for the final learned decision policy) that is very large. 
As shown in Table~\ref{tab:cost_aware_example}, let costs be \$1 for both actions in context 1, 
\$0 for action 1 in context 2, 
and \$500 for action 2 in context 2. 
Using a round-robin data collection strategy to obtain an $\epsilon$-optimal policy, which we analyzed previously, will take both actions, in both contexts, an equal number of times. However, if the probability of context 2 is very small compared to context 1, 
depending on the reward structure, 
it may be possible to learn a policy that yields a utility that is $\epsilon$-optimal by only learning the optimal action in context 1, and always taking action 1 in context 2 (the 0 cost action). Given the high cost of sampling action 2 in context 2, such an alternate data gathering strategy might be preferable if it is important to optimize the cost incurred when learning the decision policy. 

\begin{table}[t]
\centering
\begin{tabular}{rlll}
          &                                 & \multicolumn{2}{c}{Costs}                                                                  \\ \cmidrule{3-4} 
          & \multicolumn{1}{c}{Probability} & \multicolumn{1}{c}{Action 1} & \multicolumn{1}{c}{Action 2}  \\
          \toprule
Context 1 & 0.98                            & \$1                          & \$1                               \\
Context 2 & 0.02                            & \$0                       & \$500                            \\                    
\bottomrule
\end{tabular}
\vskip 0.05in
\caption{Setup for hypothetical cost-aware example.}
\label{tab:cost_aware_example}
\end{table}

Generally, we expect that a cost-aware data gathering strategy would depend on the interaction between context probabilities, cost functions, and bounds on the potential outcome (reward) ranges. This is an interesting direction for future work, and the technical innovations required will likely further increase when we use parametric assumptions on the reward models, and when budget or parity preferences ($\lambda_g > 0$) are in place.

\section{Adaptively Learning Optimal Policies}
\label{sec:online_learning}
The results from Section~\ref{sec:bounds} suggest the feasibility of solving our desired optimization problem even when the distribution of potential outcomes must be estimated from data. 
However, learning from the type of non-adaptive strategies considered above is typically not the most efficient approach to learning from data.
For instance, in our running example of providing rideshare assistance to public defender clients, if there turns out to be a group of clients with very small need and benefit from assistance, a non-adaptive learning strategy will still allocate a proportional amount of limited resources to such individuals.
In contrast, contextual bandit algorithms are often designed to maximize expected utility while learning, which typically involves estimating the potential performance of each action $a_k$ and using that information to accrue benefits.\footnote{Non-adaptive strategies are particularly useful when testing statistical hypotheses post hoc, which is most easily done with data that are independently and identically distributed across treatments.
We note that there is considerable interest in developing suitable inference methods for this latter goal using data gathered with adaptive, multi-armed bandit strategies (e.g.~\cite{hadad2021confidence,zhang2021statistical}).
} 

\begin{algorithm}[!t]
 \caption{Policy learning procedure.}
 \label{alg:online_learning}
\begin{algorithmic}[1]
 \State {\bfseries input:} 
    Actions $a_k$, 
    budget $b$, 
    parity preferences $\lambda_{g,\ell}$ and $f_{\ell}$,
    reward function $r$,
    covariate distribution $\P(X = x)$, 
    group membership function $s$,
    bandit algorithm, $n_{\text{init}}$
 \State \textbf{initialize}: Randomly treat first $n_{\text{init}}$ people
 \For{each subsequent individual $i$}
     \State \algparbox{Set $\mathcal{D}_i := \{(X_j,A_j,Y_j)\}_{j=1}^{i-1}$,
     where $X_j$, $A_j$, and $Y_j$ denote the covariates, actions, and outcomes for previously seen individuals}
    \State Estimate $\mathcal{D}_{k,x}(Y(a_k) \mid X = x)$
with a parametric family $g(x, a; \theta)$ fit on $\mathcal{D}_i$
    \If{$\varepsilon$-greedy}
        \State \algparbox{Estimate $\EE_Y[r(x, a, Y(a_k)) \mid X = x]$ and $\EE_Y[f_{\ell}(x, a, Y(a_k)) \mid X = x]$ using $g(x, a; \hat{\theta}_i)$, \\ where $\hat{\theta}_i$ is the MLE}
\ElsIf{Thompson sampling}
        \State \algparbox{Estimate $\EE_Y[r(x, a, Y(a_k)) \mid X = x]$ and $\EE_Y[f_{\ell}(x, a, Y(a_k)) \mid X = x]$ using $g(x, a; \hat{\theta}_i^*)$, \\ where $\hat{\theta}_i^*$ is drawn from the posterior of $\hat{\theta}_i$}
        \ElsIf{UCB}
        \State \algparbox{Estimate $\EE_Y[r(x, a, Y(a_k)) \mid X = x]$ using the $\alpha$-percentile of the posterior of \\ $g(x, a; \hat{\theta}_i)$ and estimate $\EE_Y[f_{\ell}(x, a, Y(a_k)) \mid X = x]$ using the $(1-\alpha$)-percentile \\
        of the posterior of $g(x, a; \hat{\theta}_i)$}
\EndIf
     \State Compute nominal budgets $b^*_{i}$ according to Eq.~\eqref{eq:budget_adjustment}
     \State \algparbox{Find solution $\pi_i^*$ of the LP
     in Section~\ref{sec:LP} with $b^*_{i}$ and the input values estimated above}
     \If{$\varepsilon\text{-greedy} \And \textsc{Bernoulli}(\varepsilon)==1$} 
        \State Take random action $A_i$ 
        according to Eq.~\eqref{eq:random_allocation}
     \Else
         \State Take action $A_i \sim \pi^*_i(X_i)$
     \EndIf
     \State Observe outcome $Y_i$
 \EndFor
\end{algorithmic}
\end{algorithm}

To more efficiently learn decision policies in the real world, we now outline our procedure to integrate the LP formulation from Section~\ref{sec:problem} with three common contextual bandit approaches: 
$\varepsilon$-greedy, Thompson sampling, and upper confidence bound (UCB), as described in Algorithm~\ref{alg:online_learning}.
For simplicity, we assume knowledge of the covariate distribution $\DD(X)$, which is often easily obtained from historical data, even in the absence of past interventions.
If historical data are not available, the covariate distribution can instead be estimated from the sample of individuals observed during the decision-making process.

At a high level, at each step $i$, our 
$\varepsilon$-greedy approach 
first estimates 
$\mathcal{D}(Y(a_k) \mid X)$
using the maximum likelihood estimates of a chosen parametric family.
We next use these estimates to find the optimal policy $\pi_i^*$ with our LP.
Then, with probability $1-\varepsilon$, 
we treat the $i$-th individual according to $\pi_i^*$; 
otherwise, with probability $\varepsilon$, we take action $a_k$ with a probability set to meet our budget requirements in expectation.
Our Thompson sampling approach maintains a posterior over the parameters of a model of the potential outcomes $\mathcal{D}(Y(a_k) \mid X)$, samples from this posterior, uses the posterior draw to construct the inputs for our LP, yielding a policy $\pi_i^*$, and then treats the $i$-th individual according to $\pi_i^*$.
Finally, under our UCB approach,
we compute $\pi_i^*$ by solving the LP with optimistic estimates of $r$ and the parity penalities (e.g., using the 97.5th percentile of the posterior of the former and the 2.5th percentile of the latter). 
\subsection{Simulation Study}
\label{sec:simulation}

\begin{figure}[t]
\begin{center}
\centerline{\includegraphics[width=0.6\columnwidth]{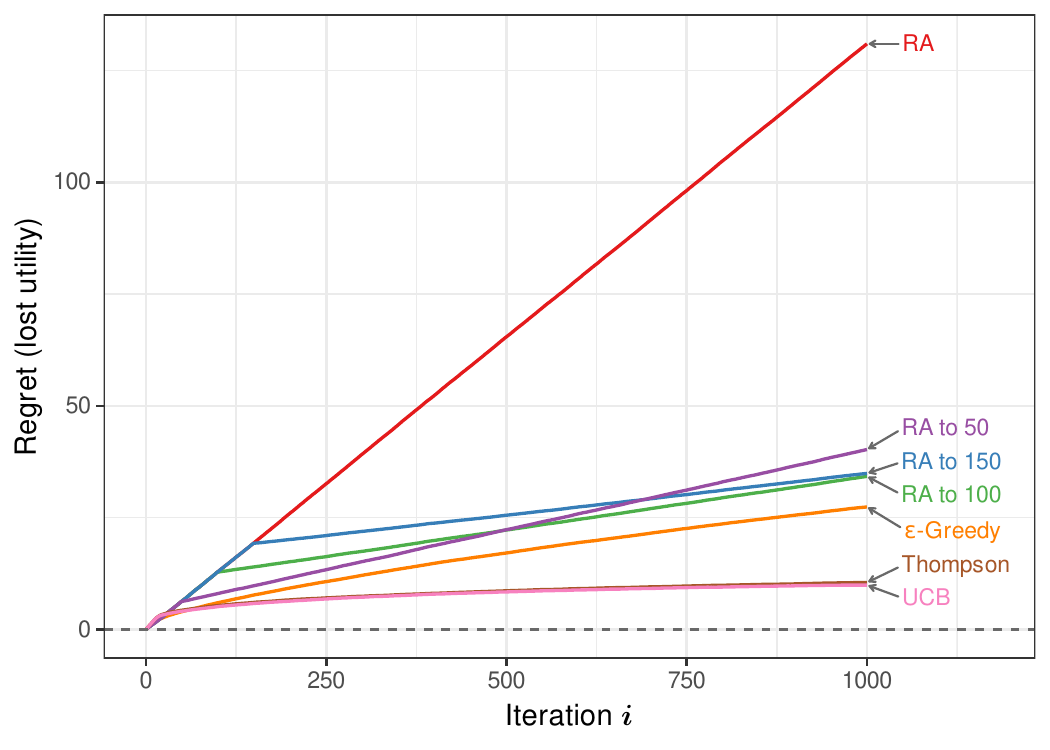}}
\caption{
Mean regret, across 2,000 simulations, incurred by different learning approaches.
We define regret here as the difference between the observed utility and the utility obtained by an oracle during the same experiment.
Values are tightly estimated at each $i$, 
with the 95\% confidence interval no more than 1.1 units off the estimate,
so we omit uncertainty bands for this figure.
We note that the three bandit approaches---$\varepsilon$-greedy, Thompson sampling, and UCB---incur substantially less regret than random assignment (RA).
It is possible to reduce the regret incurred from RA by stopping randomization early, and following the optimal estimated policy from that point forward.
However, these stop-early RA approaches produce worse policies than other approaches (Figure~\ref{fig:pct_of_optimal}).
}
\label{fig:regret}
\end{center}
\end{figure}

\begin{figure}[t]
\begin{center}
\centerline{\includegraphics[width=0.6\columnwidth]{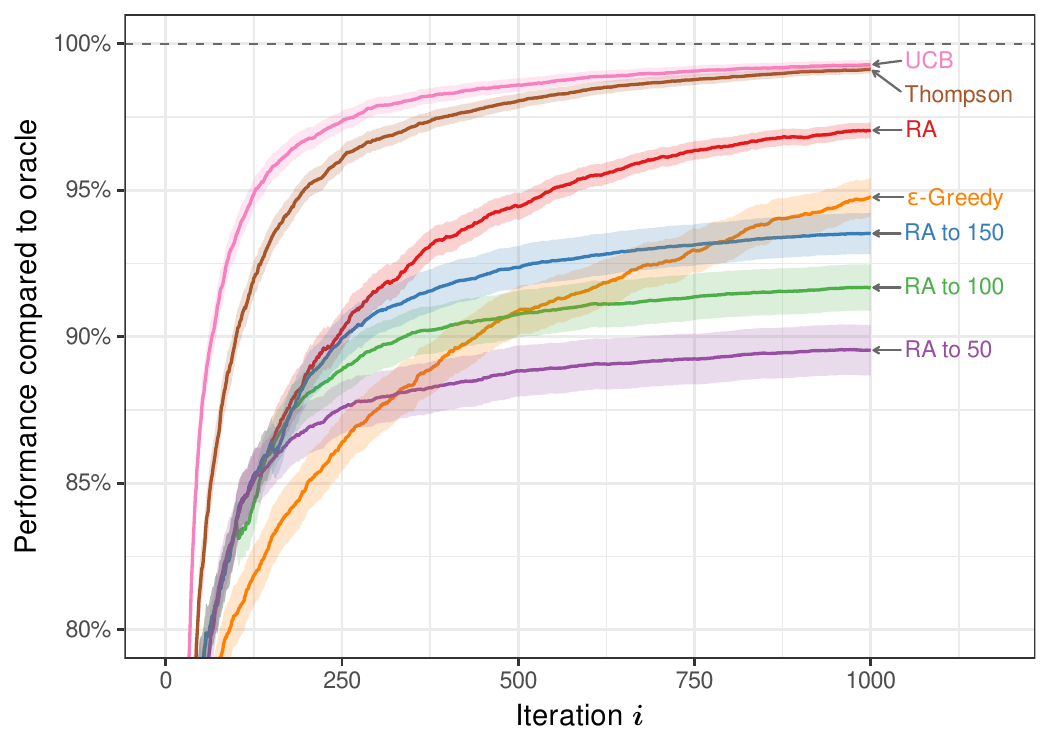}}
\caption{
Mean performance, across 2,000 simulations, of optimal policies estimated with data available at each iteration $i$.
Performance is defined as the additional utility obtained by a policy over a baseline of no treatment for all individuals, with 100\% indicating this quantity for the oracle policy.
Uncertainty bands represent 95\% intervals around the mean.
UCB and Thompson sampling generate policies that are better than random assignment (RA) at any given iteration~$i$.
In contrast, the $\varepsilon$-greedy approach and the stop-early versions of RA generate policies that are slower to (or may never) reach near-oracle performance.
}
\label{fig:pct_of_optimal}
\end{center}
\end{figure}

To evaluate our learning approach above, we conducted a simulation study using data on a sample of clients served by the Santa Clara County Public Defender Office.
In this example, clients can receive one of three mutually exclusive treatments $a_k$: 
round-trip rideshare assistance, 
a transit voucher, 
or no transportation assistance.
We fix our budget to \$50,000 and limit our population to 1,000 clients, resulting in an average per-person budget of \$5.
In line with many government pilot programs,
we assume that this funding is dedicated to learning suitable policies,
and that our hypothetical public defender would be able to provision separate funding later  
to operate a more permanent program.
We set the cost of rides to \$5 per mile.
We limit the client population to white and Vietnamese individuals to reflect our running example.
The utility of a policy is described by
Eq.~\eqref{eq:utility}, 
where we set $r(x, a, y) = y$, $f(x, a, y) = c(x, a)$, and $\lambda = 0.0006$.
This choice yields an oracle policy 
that balances between 
maximizing appearances and achieving parity in per-capita spending across groups.
The data generating process for this population and additional experiment parameters are described in detail in Appendix~\ref{appx:exp_details}.

We compare our contextual bandit approaches against several baselines.
First, we compare to non-adaptive random assignment (RA), in which treatment is randomly selected (in accordance with the budget) throughout the entire learning phase.
The simplicity and versatility of RA makes it a common strategy for learning optimal policies.
We also include partially adaptive variations on this approach, where we run RA on the first $n$ individuals, 
and then follow the optimal policy estimated at individual $n$ for the rest of the sample, similar to explore-first strategies. 
We compare all approaches against an oracle that can observe the true appearance probabilities.

We repeat this evaluation 2,000 times each on 1,000 randomly selected individuals from our dataset, and compare the performance of all approaches using two different metrics.
Our main two bandit approaches---Thompson sampling and UCB---significantly reduce regret when compared to non-adaptive and partially adaptive approaches during the learning phase (Figure~\ref{fig:regret}).
Our bandit approaches also learn policies that, if used for future populations, would outperform non-adaptive and partially adaptive approaches (Figure~\ref{fig:pct_of_optimal}).
In contrast to our two main bandit algorithms, the $\varepsilon$-greedy approach also manages to reduce regret, but is slower to learn a near-oracle policy.
RA and its variations illustrate the limits of the conventional randomized approach.
For example, it is possible to learn a near-oracle policy using classic RA, but this incurs substantial regret during the learning phase.
Though it is possible to reduce this regret by ending RA early, these alternatives learn poorer-performing policies.

By design, the bandit methods discussed above reduce spending disparities during the course of the simulation. 
We demonstrate this by comparing our main simulation to an alternate set of simulations where $\lambda_g = 0$ (Table~\ref{tab:mean_spending}).
For example, with a choice of $\lambda_g = 0.004$, reflecting a mild preference for more equal spending,
we observe that UCB methods spent \$2.21 less on Vietnamese clients than the \$5 population average (i.e., the target budget).
In contrast, with a choice of  $\lambda_g = 0$ (i.e., preferring policies that simply aim to maximize appearances), UCB methods spent \$3.36 less on Vietnamese clients compared to the population average.

\begin{table}[t]
\centering
\begin{tabular}{@{}lll@{}}
\toprule
& \multicolumn{2}{@{}c}{Vietnamese spending disparity}  \\
\cmidrule(r){2-3}
& With penalty & No penalty       \\ 
Method & ($\lambda_g = 0.004$) & ($\lambda_g = 0$) \\ \midrule
UCB & -\$2.21 & -\$3.36 \\ 
Thompson & -\$1.21 & -\$2.19 \\ 
$\epsilon$-Greedy & -\$1.29 & -\$2.38 \\ 
\bottomrule
\end{tabular}
\vskip 0.2in
\caption{
Mean spending disparities by method for Vietnamese clients across 2,000 experiments, including both the main set of simulations (where $\lambda_g = 0.004$) and an alternative set of simulations (with identical parameters to the main set, except where $\lambda_g = 0$).
Disparities are calculated by comparing average spending on Vietnamese individuals to the \$5 average spending on all individuals (i.e., the target budget).
Note that spending disparities are approximately \$1 larger when $\lambda_g = 0$, verifying that the bandit methods we employ in our simulation learn to reduce spending disparities to maximize the policymaker's utility.
}
\label{tab:mean_spending}
\end{table}

\begin{figure}[t]
\begin{center}
\centerline{\includegraphics[width=0.6\columnwidth]{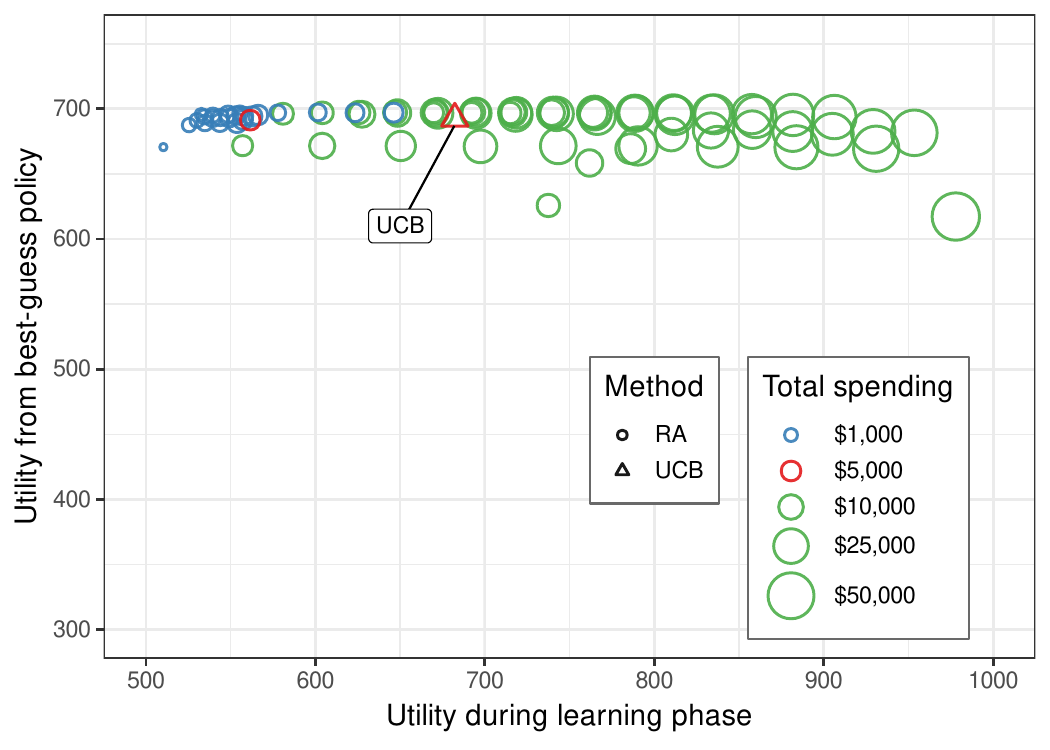}}
\caption{
The effect of varying spending on
outcomes of interest.
Each circle represents average outcomes across 125 simulations of random assignment (RA) with a given allocation.
For the sake of comparison, we also included the average outcome across 2,000 simulations of UCB 
from earlier in this section,
represented here by a single triangle. 
Each glyph is sized by total spending,
with color indicating if these approaches spent less (blue),
equivalent (red),
or more (green)
money compared to the methods discussed earlier in this section.
We find that varying spending mostly affects the utility observed during the learning phase,
but has little effect on the quality of the final policy learned. 
Random assignment strategies that save money (when compared to UCB) do not achieve as much utility during the leraning phase,
though both approaches result in similar-quality policies.
}
\label{fig:spending_variations}
\end{center}
\end{figure}

The bandit approaches we discuss above aim to maximize utility during the learning phase,
but do not explicitly try to minimize money spent during learning. 
As discussed in Section~\ref{sec:sample_bounds_costs},
it is possible that alternate approaches may spend less while achieving similar outcomes.
One could imagine a learning strategy in which nearly all participants were offered the no-cost treatment,
with only a small number offered a costly treatment (a ride or transit voucher).
With these alternate approaches, 
we may be able to learn the structure of appearance behavior by using the no-cost treatment for most participants,
and then learn the impact of costly treatments with a small number of remaining participants.
To explore such alternate approaches empirically, 
we considered a range of policies that assign the first 1,000 clients in each experiment
to one of our three treatment arms in different random allocations.
For example, 
one variation randomly allocated rides to only 2\% of clients,
and transit vouchers to only 2\% of clients, 
with the remaining 96\% of clients receiving the no-cost control action.
Another variation randomly allocated rides for 10\% of clients, and vouchers for 40\% of clients, 
with remaining 50\% of clients receiving the no-cost control action.
Additional details describing this simulation are included at the end of Appendix~\ref{appx:exp_details}.

We show the results of this exercise in Figure~\ref{fig:spending_variations}.
This plot compares three dimensions on which we evaluate each policy:
first, the utility achieved during the learning phase;
second, the quality of the policy learned by the end of the phase;
and third, the total amount spent during learning.
We see that varying spending mostly affects the utility observed during learning,
with more expensive allocations resulting in higher utility during learning.
Spending appears to have little impact on the quality of the final policy learned.
For the sake of comparison,
Figure~\ref{fig:spending_variations} also includes UCB results from the simulations at the beginning of this section.
Among the spending variations we tested, no variation spent less money than UCB while achieving similar utility during the learning phase.
This suggests that UCB can be a cost-effective approach to maximizing utility while learning a high-quality policy.

\section{Discussion}
\label{sec:discussion}
We have outlined a consequentialist framework for equitable algorithmic decision-making.
Our approach foregrounds the role of an expressive utility function that captures preferences for both individual- and group-level outcomes.
In this conceptualization, we explicitly consider the inherent trade-offs between competing objectives in many real-world problems.
For instance, in our running example of allocating transportation assistance to public defender clients, 
there is tension between maximizing appearance rates and equalizing spending across groups.
Popular rule-based approaches to algorithmic fairness---such as enforcing spending parity or equal false negative rates 
across groups---implicitly balance 
these competing objectives in ways that
may be at odds with the actual preferences of stakeholders.
Our approach, in contrast, requires one to confront the consequences of difficult trade-offs, and, in the process, may help one improve those decisions.

For a rich class of utility functions, 
we showed that one can efficiently learn optimal decision policies
by coupling ideas from the contextual bandit and optimization literatures. 
For example, with our UCB-based algorithm, we do so by repeatedly solving a linear program under optimistic estimates of the potential outcomes of actions.
In an empirically grounded simulation study, we showed that this strategy can outperform common alternatives, including learning through random assignment or acting greedily based on the available information.

Our learning algorithm requires access to a well-specified utility function that reflects stakeholder preferences.
In practice, inferring this utility is a complex task in its own right, but the illustrative survey that we conducted shows how one can begin to operationalize this task.
Challenges may arise from an unwillingness to explicitly state preferences for trade-offs involving sensitive considerations like demographic parity.
There are, however, several established techniques to elicit multi-faceted preferences less directly. 
One family of approaches selects pairs of similar realistic scenarios, asks stakeholders to pick their preferred outcome, and infers their preferences from these choices~\citep{koenecke2023,linpreference,furnkranz2010preference, chu2005preference, jung2019algorithmic}.

Another challenge---particularly relevant in the dynamic setting---is accounting for delayed outcomes.
In our running example, we may choose to offer rideshare assistance to a client days or weeks before their appointment date.
As a result, there may be large gaps between when an action is taken and when we observe its outcome. 
Thompson sampling methods have been observed to be more robust to delayed outcomes than upper confidence bound strategies in contextual bandit scenarios~\citep{chapelle2011empirical}. Another way to address this issue is through the use of \textit{proxies} or \textit{surrogates}, in which intermediate outcomes are used as a temporary stand-in for the eventual outcome of interest~\citep{athey2019surrogate}.
For example, with rideshare assistance to clients, one might use intermediate responses (like a client's confirmation to attend their appointment) as a proxy for appearance.
A third approach might be to reduce the budget for costly actions, effectively limiting the resources spent while waiting to observe outcomes.

In addition to the above technical considerations,
we note some practical limitations in providing transportation to public defender clients with upcoming court dates.
First, in many circumstances 
policymakers may not be legally permitted to explicitly use race, ethnicity, or other protected attributes when deciding how to allocate limited resources.
These policymakers may instead focus on other attributes, like geography or socioeconomic status, which may be legally or socially more permissible.
Second, our motivating example presupposes that resources are too limited to aid the entire population of interest. 
If policymakers had enough funding available to assist an entire population, it may not make sense to even consider equalizing per-capita spending across groups of interest, 
given that everyone would receive transportation assistance.
Finally, though this study emphasizes the potential benefits of rideshare assistance for those who have mandatory court dates (e.g., one potential benefit is avoiding time in jail),
a simpler and more effective policy for reducing jail time may be to discourage judges from issuing bench warrants if clients fail to appear in court.
Though in isolation this policy might result in lower appearance rates, it could be accompanied by other assistance to offset this adverse outcome, including text message reminders, social services, or rideshare assistance as we describe here~\citep{chohlas2023automated,fishbane2020behavioral,zottola2023court}.

Algorithms impact individuals both through the decisions they guide and the outcomes they engender.
Looking forward, we hope our work helps to elucidate the subtle interplay between actions and consequences, and, in turn, furthers the design and deployment of equitable algorithms.

\section{Acknowledgements}
We thank Johann Gaebler, Jonathan Lee, Hamed Nilforoshan, Julian Nyarko, and Ariel Procaccia for helpful comments.
We also thank colleagues at the Santa Clara County Public Defender Office for their assistance, 
including Molly O'Neal, Sarah McCarthy, Terrence Charles, and Sven Bouapha.
This work was supported in part by grants from the Stanford Impact Labs and the Stanford Institute for Human-Centered Artificial Intelligence. 
Code to replicate our analysis is available online at: \url{https://github.com/stanford-policylab/learning-to-be-fair}.

\bibliography{refs}

\clearpage
\appendix

\renewcommand\thetable{\thesection.\arabic{table}}
\renewcommand\thefigure{\thesection.\arabic{figure}}
\renewcommand{\theequation}{\thesection.\arabic{equation}}
\setcounter{table}{0}
\setcounter{figure}{0}
\setcounter{equation}{0}

\section*{Appendices}

\section{Absolute Value in an LP Objective}
\label{appendix:LP}
If $(v^*, w^*)$ is a solution to the LP in Eq.~\eqref{eq:abs-lp},
then we claim $v^*$ is a solution to the original optimization problem in Eq.~\eqref{eq:abs-opt}.
Let $\text{OPT}_{\text{abs}}$
and $\text{OPT}_{\text{LP}}$
denote the optima of Eqs.~\eqref{eq:abs-opt} and \eqref{eq:abs-lp} above.
Now, since $w_{g,\ell} = |\beta_{g,\ell}^Tv|$ satisfies the LP constraints,
$\text{OPT}_{\text{abs}} \leq \text{OPT}_{\text{LP}}$.

Conversely, because the LP objective function decreases in $w_{g,\ell}$,
if $\beta_{g,\ell}^Tv^* \geq 0$, 
then $w_{g,\ell}^* = \beta_{g,\ell}^Tv^*$ (since $\beta_{g,\ell}^Tv^* \leq w_{g,\ell}$, and the other two constraints are immediately satisfied in this case).
On the other hand, if $\beta_{g,\ell}^Tv^* \leq 0$, 
then $w_{g,\ell}^* = -\beta_{g,\ell}^Tv^*$ (since $\beta_{g,\ell}^Tv^* \geq -w_{g,\ell}$).
Thus, in either case, $w_{g,\ell}^* = |\beta_{g,\ell}^Tv^*|$,
which implies that 
\begin{equation*}
\text{OPT}_{\text{abs}} = \text{OPT}_{\text{LP}} = \alpha^T v^* - \sum_{g,\ell} \lambda_{g,\ell} |\beta_{g,\ell}^Tv^*|.
\end{equation*}

\section{Group-Specific Threshold Rules}
\label{sec:thresholds}

The LP described in Section~\ref{sec:LP} yields a solution to our general decision-making problem, 
with an arbitrary number of treatment arms and a potentially complex utility function. 
Here we show that in the case of $K=2$ treatments (e.g., with the options corresponding to whether or not one provides rideshare assistance) and parity penalties only on cost, 
optimal decision policies can be expressed in a simple, interpretable form.
Moreover, for a reward function $r$ that decomposes into aggregate and individual components---as in Eq.~\eqref{eq:r}---we can view  optimal policies as group-specific threshold rules.

\begin{theorem}
\label{thm:threshold}
In the setting of Section~\ref{sec:problem}, suppose $K=2$.
Further assume that we only impose parity penalties on the cost: $f_1(x, a, y) = c(x, a)$, where 
$c(x, a_0) = 0$ and
$c(x, a_1) > 0$.
Finally, suppose
$|s(x)| = 1$ (i.e., $\mathcal{G}$ partitions $\mathcal{X}$)
and that $\Delta(x) > 0$, where 
\begin{align*}
    \Delta(x) = \EE_Y[r(x, a_1, Y(a_1)) - r(x, a_0, Y(a_0)) \mid X = x].
\end{align*}
Then, for group-specific constants $t_g$ and $p_g$, there exists an optimal decision policy $\pi^*$ of the form
\begin{equation}
\label{eq:threshold}
\begin{aligned}
    \Pr(\pi^*(x) = a_1) = \left\{
    \begin{array}{ll}
    1 & \Delta(x) / c(x, a_1) >  t_{s(x)} \\
    p_{s(x)} & \Delta(x) / c(x, a_1) = t_{s(x)} \\
    0 & \textup{otherwise}.
    \end{array}
    \right.
\end{aligned}
\end{equation}
\end{theorem}

\begin{proof}
{Proof.}
We start by rewriting the utility $U(\pi)$ as
\begin{align*}
U(\pi) &= \sum_x \EE_Y[r(x, a_0, Y(a_0)) \mid X = x] \cdot \Pr(X = x) \\
& \hspace{1cm} + \sum_x \Delta(x) \cdot \Pr(\pi(x) = a_1) \cdot \Pr(X = x) \\
& \hspace{1cm} - \sum_{g \in \mathcal{G}} \lambda_g \delta_g(\pi),
\end{align*}
where
\begin{equation*}
    \delta_g(\pi) = | \EE_{X}[c(X, \pi(X)) \mid g \in s(X)] - \EE_{X}[c(X, \pi(X))] |.
\end{equation*}
Now, for any policy $\pi$, we construct a threshold policy $\tilde{\pi}$ of the form in Eq.~\eqref{eq:threshold} by assigning to action $a_1$ those $x$ in each group $g$ having the largest values of $\Delta(x) / c(x, a_1)$ such that 
\begin{equation*}
\EE_{X}[c(X, \tilde{\pi}(X)) \mid g \in s(X)] =
\EE_{X}[c(X, \pi(X)) \mid g \in s(X)].
\end{equation*}
By construction, $\delta_g(\tilde{\pi}) = \delta_g(\pi)$, and 
\begin{equation*}
\sum_x \Delta(x) \Pr(\tilde{\pi}(x) = a_1) \Pr(X = x)
\geq
\sum_x \Delta(x) \Pr(\pi(x) = a_1) \Pr(X = x).
\end{equation*}
Consequently, $U(\tilde{\pi}) \geq U(\pi)$,
establishing the result.
$\hfill \square$
\end{proof}

\begin{table}[t]
\centering
\resizebox{\textwidth}{!}{\begin{tabular}{c|c|c|ccc|ccc}
$X$   & $\Pr(X = x)$ & $\EE[Y(0) | X = x]$ & $\EE[Y(1) | X = x]$ & $c(x, a_1)$ & $ \Delta(x) / c(x, a_1) $ & $\EE[Y(2) | X = x]$ & $c(x, a_2)$ & $ \Delta(x) / c(x, a_2) $ \\
\toprule
$x_1$ & 0.1          & 0.1                 & 0.6                 & \$10        & 0.05                      & 0.3                 & \$1         & \cellcolor[HTML]{C0C0C0}0.2                       \\
$x_2$ & 0.9          & 0.1                 & 0.2                 & \$10        & 0.01                      & 0.12                & \$1         & \cellcolor[HTML]{C0C0C0}0.02                     
\end{tabular}}
\caption{Setup for counterexample.}
\label{tab:counterexample}
\end{table}

In the theorem above, we assume $K=2$, which yields a simple threshold solution for the optimal policy.
In general, with $K > 2$, the structure of the optimal policy can be more complicated. 
As a counterexample to the above, suppose $K=3$, with a no-cost baseline action $a_0$, and two costly actions, $a_1$ and $a_2$. 
We further imagine a population with a single group (i.e., $|\mathcal{G}| = 1$) with individuals in two contexts characterized by attributes $x_1$ and $x_2$,
and a utility $U(\pi) = \EE_{X, Y}[Y(\pi(X))]$.
Table~\ref{tab:counterexample} lists the responsiveness of each type of individual to each of the three possible actions, the costs of the two costly actions, and the relative benefit per dollar of the two costly actions over the free baseline action.
In this setup, we set $b = \$1$, i.e., we can spend, on average, one dollar per person.

For both types of individuals in this example, action $a_2$ has the highest relative benefit per dollar over the baseline action $a_0$ (highlighted in Table~\ref{tab:counterexample} in gray).
As such, one intuitive strategy $\pi$ is to treat every individual with $a_2$, exhausting our budget, yielding 
\begin{align*}
    U(\pi) &= (0.3 \cdot 0.1) + (0.12 \cdot 0.9) = 0.138.
\end{align*}
However, in this case, a better strategy $\pi^*$ is to assign action $a_1$ to all individuals of type $x_1$, and to assign $a_0$ to all individuals of $x_2$, exhausting our budget and yielding
\begin{align*}
    U(\pi) &= (0.6 \cdot 0.1) + (0.1 \cdot 0.9) = 0.15.
\end{align*}
In this setup, even though $a_2$ has the highest relative benefit \textit{per dollar}, it is low cost in absolute terms, meaning that we have to treat individuals of both types to exhaust our budget, including individuals of type $x_2$ who see little benefit to the treatment.
As a result, it is better to exhaust one's budget on action $a_1$ with individuals of type $x_1$, which sees a higher return on average than treating the entire population with $a_2$.
This type of scenario demonstrates the need for more complicated solutions to identifying an optimal policy, 
justifying the use of a linear program (or other similar approaches).

More generally, if we were to allow arbitrary utility functions, then finding an optimal decision policy is NP-hard,
as we show below.

\begin{proposition}
\label{prop:hardness}
If we allow arbitrary utility functions $U$ in Eq.~\eqref{eq:opt-problem}, then finding an optimal policy is NP-hard.
\end{proposition}

\begin{proof}
{Proof.}
We reduce to the NP-hard subset sum problem. Given integers 
$x_1, \dots, x_n$, 
consider the policy optimization problem for $K=2$ actions and no budget constraints (i.e., $c(x,a_k) = b = 1$), with utility
\begin{equation*}
U(\pi) = \left\{ 
\begin{array}{cl}
     1 & A(\pi) \neq \emptyset \ \land \displaystyle \sum_{i \in A(\pi)} x_i = 0 \\ 
     0 & \text{otherwise}
\end{array}
\right.
\end{equation*}
where $A(\pi) = \{i : \Pr(\pi(x_i) = a_1) = 1\}$.
Then $\max_{\pi} U(\pi) = 1$ if and only if there exists a non-trivial subset of the integers $\{x_1, \dots, x_n\}$ that sums to zero, establishing the claim.
\end{proof}

\section{Details for Section~\ref{sec:tradeoffs}}
\label{appx:trade_offs}

For clients in our synthetic population, distance from court is determined by folded normal distributions with density: 
\begin{equation*}
g_{\mu,\sigma}(x) = \left\{ 
\begin{array}{cl}
     \frac{2}{\sigma \sqrt{2\pi}}~\exp{-\frac{(x - |\mu|)^2}{2\sigma^2}} & \text{if } x \geq 0,
     \smallskip \\
     0 & \text{otherwise.}
\end{array}
\right.
\end{equation*}
White clients in this population have distances given by the following density:
\[
f_w(x) = g_{\mu_1, \sigma_1}(x),
\] where $\mu_1 = 2$ and $\sigma_1 = 1$.
For Black clients, distance from court is determined by a mixture distribution with a density given by:
\[
f_b(x) = 0.25 \cdot g_{\mu_2, \sigma_2}(x) + 0.75 \cdot g_{\mu_3, \sigma_3}(x),
\]
where $\mu_2 = 1, \sigma_2 = 1, \mu_3 = 10,$ and $\sigma_3 = 5$.

\section{Details for Figure~\ref{fig:survey_results}}
\label{appendix:survey_prompt}

Figure~\ref{fig:survey_results} showed our survey population's preferred allocations of rides in a hypothetical city, 
split by U.S. political party affiliation.
In Figure~\ref{fig:addl_survey_results}, we include two similar plots,
split by self-identified gender and race/ethnicity.
As in Figure~\ref{fig:survey_results},
we observe that members of each group have a wide range of preferences for the allocation of spending,
suggesting that there is no universal optimal tradeoff 
for this scenario and other similar policy domains.\\

\noindent
We solicited respondents for our survey on Prolific with the following prompt:

\begin{quote}
\begin{it}
{\parindent0pt 
    \textbf{Design a Transportation Assistance Policy}\\

    In this study, we will ask you for your opinion on how you would distribute rideshare assistance to help people get to court.\\
    
    You will also be asked to answer standard demographic questions at the end of the survey. \\
    
    Your responses are completely anonymous.
}
\end{it}
\end{quote}

\noindent
Respondents were then presented with the following introduction at the start of their survey:

\begin{quote}
\begin{it}
{\parindent0pt 
\textbf{Court Transportation Survey}\\

Many people who are charged with a crime must appear in court. If they miss court, they are often put in jail. Some of these people miss court because it is difficult for them to get to the courthouse. For instance, they might not be able to afford transportation. And so, by providing those people with free rides to court, it is possible to help them avoid jail time. In this survey, we are interested in your views on how to distribute free rides to court.\\

Imagine a fictional U.S. city called ``Metropolis.'' In Metropolis, Black people typically live farther away from the courthouse than White people. Because long-distance rides are more expensive than short-distance rides, a round-trip ride to court typically costs \$20 for White people and \$80 for Black people.
}
\end{it}
\end{quote}

\noindent
Finally, respondents are shown the following summary alongside the graphic in Figure~\ref{fig:survey-graphic}:

\begin{quote}
\begin{it}
{\parindent0pt 
Now imagine that you live in Metropolis. The city has \$50,000 to provide people with free rides to court. How would you distribute the free rides across the city's Black and White populations? 50\% of the city's inhabitants are Black, and 50\% are White.\\

The figure below shows five options for spending your city's budget. Each one has a different number of rides for White and Black people, but all cost \$50,000. \$50,000 is the annual budget for the next 3 years.\\

Please select your preferred option to distribute the \$50,000. Your responses will remain anonymous. 
}
\end{it}
\end{quote}

\begin{figure}
    \vspace{1.2cm}
    \centering
    \begin{subfigure}[t]{\columnwidth}
        \centering
        \includegraphics[width=0.8\columnwidth]{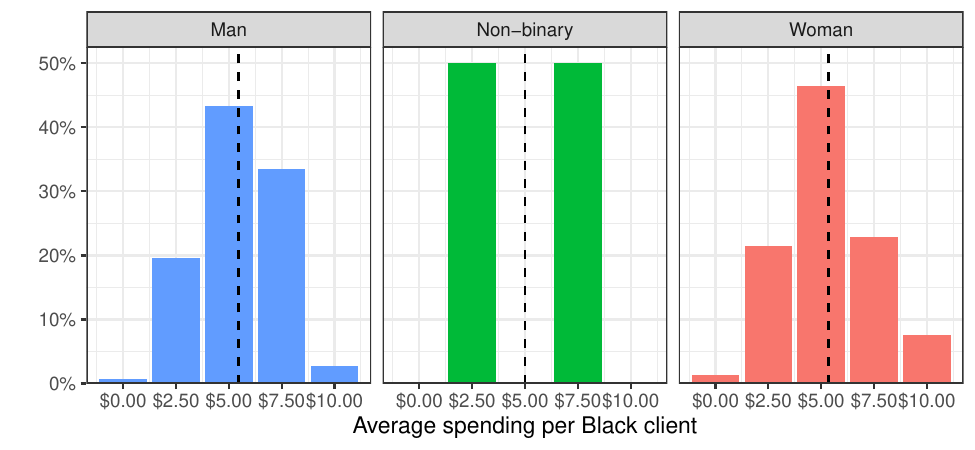}
\label{fig:addl_survey_results_bygender}
    \end{subfigure}
    \vskip 0.1in
    \begin{subfigure}[t]{\columnwidth}
        \centering
        \includegraphics[width=0.8\columnwidth]{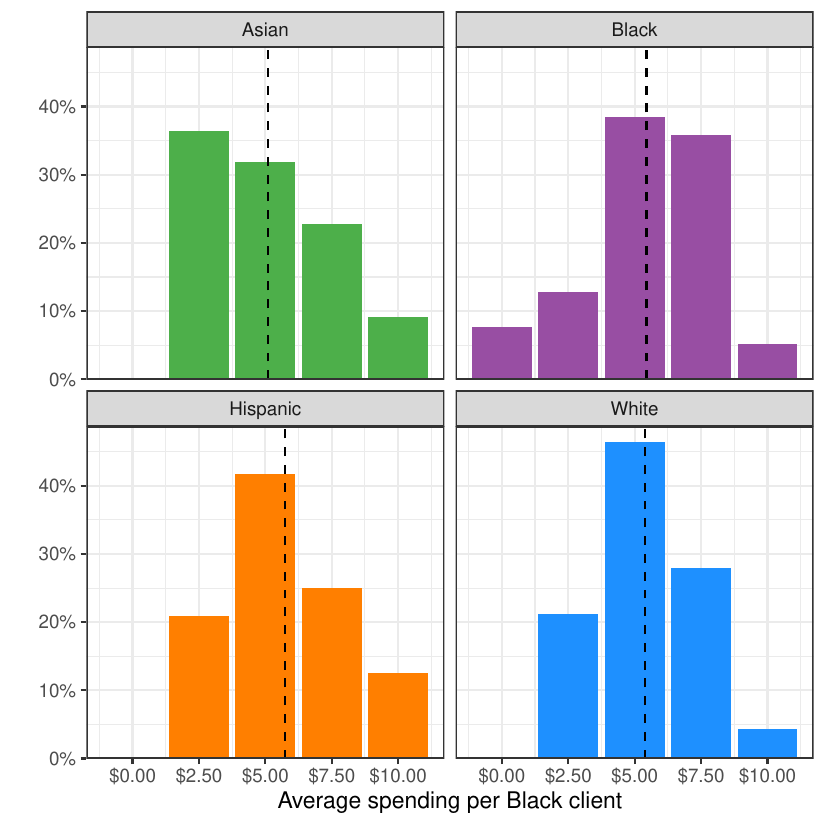}
\label{fig:addl_survey_results_byethnicity}
    \end{subfigure}
    \vskip 0.1in
    \caption{
    Data from the survey discussed in Section~\ref{sec:tradeoffs}.
    Here, we show the distribution of respondent preferences
    for average spending allocated to Black individuals,
    split by gender (top) 
    and race/ethnicity (bottom).
    }
    \label{fig:addl_survey_results}
\end{figure}

\section{Proofs for Section~\ref{sec:sample_bounds}}
\label{appx:proofs_sample_bounds}
In this section we use the shorthand notation $p_x = \P(X=x)$, $\pi_{x k} = \P(\pi(x)=a_k)$, and $r_{xk} = f(x, k)$.

We now use Lemma~\ref{lemma:rew_to_val_err_bound} to prove sample complexity bounds for  our reward settings.

\begin{theorem}[Restatement of Theorem~\ref{thm:rct_tabular}]
Assume the reward is tabular and the costs are known. 
Suppose we collect $n$ samples in a round-robin fashion (i.e., for each context $x$, select the least-sampled action $a_k$ in that context, breaking ties arbitrarily).
Then for $\epsilon > 0$, $\delta > 0$, $\lambda_g \geq 0$,
and
\begin{equation*}
  n  \geq  16\sigma^2 \frac{|X||A|}{\epsilon^2}\ln\frac{4|X||A|}{\delta} \ln \frac{2|X|}{\delta}
\end{equation*}
we have $\P(U(\pistar) - U(\pihat) < \epsilon) > 1- \delta$.
\end{theorem}

\begin{proof}{Proof.}
We proceed by first showing that if we observe, for each context-action, a sufficient number of samples then we can guarantee that a near-optimal policy is computed with high probability. Since lower probability contexts contribute less to the overall utility, a larger estimation error on these lower probability contexts is permissible, and the resulting number of samples needed per context will be proportional to the probability of the context. Since we cannot select contexts to sample (rather they arrive from a known stochastic distribution), we then provide a sufficient bound on the number of total observations (across all contexts) which guarantees that we will observe the sufficient number of samples per state-action computed in the first part.

Starting with the first part, we observe that the sample mean estimator of reward has distribution
\begin{equation}
    \rhat_{xk} = \frac{1}{n_{xk}}\sum_{i=1}^{n_{xk}} R_{xk, i} \sim \subGaussian(\frac{\sigma^2}{n_{xk}^2}).
\end{equation} Recall $p_x$ is our shorthand for the probability of context $x$. First note given $(\epsilon / (2 \sqrt{p_x |X|}) )$-accurate estimates of the reward model for each context $x$ and action $a$,  
Lemma \ref{lemma:rew_to_val_err_bound} shows the resulting utility error will be upper bounded by $\epsilon$: 
\begin{align}
    U(\pistar) - U(\pihat)
        \leq 2 \sum_x p_x \max_k |r_{xk} - \rhat_{x k}| < 2 \sum_x p_x \frac{\epsilon}{2 \sqrt{p_x |X|} } = \frac{\epsilon}{\sqrt{|X|}} \sum_x \sqrt{p_x} \leq \epsilon. 
\end{align}
In the inequality we substituted in the assumed bound on the reward estimation accuracy, and the last inequality uses the  Cauchy-Schwarz inequality.

Using Hoeffding's concentration inequality we can ensure that after $n_{xk}$ samples for each context-action pair, our estimated rewards $\rhat$ satisfy the desired accuracy bound $|r_{xk} - \rhat_{x k}| \leq  \frac{\epsilon}{2 \sqrt{p_x |X|}}$ with probability at least $ 1-\frac{\delta}{2 |X||A|}$, where 
\begin{equation}
n_{x k} = \frac{8\sigma^2 p_x |X|}{\epsilon^2}\log\frac{4|X||A|}{\delta}. \label{eqn:nx}
\end{equation}
Therefore using a union bound over all contexts and actions, we obtain a high probability bound on the utility loss from using the estimated reward models: 
 $\P( U(\pistar) - U(\pihat) \leq \epsilon) > 1-\delta/2$. 

Assuming that actions are sampled in a round robin fashion per context, this means 
\begin{equation}
n_{x} = \frac{8\sigma^2 p_x |X||A|}{\epsilon^2}\log\frac{4|X||A|}{\delta}
\end{equation} 
samples are needed per context. 

We now prove that for sufficiently large $n$,  with high probability, at least $n_x$ samples will be observed per context.

Here it will be helpful to index the contexts which we will label as $x_1,x_2,\ldots,x_{|X|}$. Consider some number $\tilde{n}$ of total data samples , which will be composed of $(\tilde{n}_{x_1},\tilde{n}_{x_2},\ldots,\tilde{n}_{x_{|X|}})$ samples of $(x_1,x_2,\ldots)$ respectively. Then we wish to bound the probability of failure of any of the contexts not to receive the required number of observations:
\begin{equation}
    P(\tilde{n}_{x_1} \leq n_{x_1} \cup \tilde{n}_{x_2} \leq n_{x_2} \cup \ldots \tilde{n}_{x_{|X|}} \leq n_{x_{|X|}} | \tilde{n}) \leq \delta'
\end{equation}
Note we can upper bound the left hand side by the sum of the probability of the individual events. 
\begin{equation}
    P(\tilde{n}_{x_1} \leq n_{x_1} \cup \tilde{n}_{x_2} \leq n_{x_2} \cup \ldots \tilde{n}_{x_{|X|}} \leq n_{x_{|X|}} | \tilde{n}) \leq \sum_{i=1}^{|X|}  P(\tilde{n}_{x_i} \leq n_{x_i} | \tilde{n} )
\end{equation}
We now consider the probability of failure for a particular context $x_i$. For each of the $\tilde{n}$ samples, we can consider a Bernoulli trial, where with probability $p_{x_i}$ we observe context $x_i$, else we observe one of the other contexts. Lemma 56 of \cite{lithesis2009} proves that less than $n_{x_i}$ observations of context $x_i$ will occur, with probability at most $\frac{\delta}{2|X|}$  if the number of samples are at least 
\begin{eqnarray}
    \tilde{n} &\geq&  \frac{2}{p_{x_i}} (n_{x_i} + \ln \frac{2|X|}{\delta}) \label{eqn:nc}
\end{eqnarray}
Note that $n_{x_i} + \ln \frac{2|X|}{\delta} \leq n_{x_i} \ln \frac{2|X|}{\delta}$ if $\delta <= 0.5$ and $n_{x_i} >= 2$ and there are two or more contexts. (If there is only a single context, this is simply a multi-armed bandit problem and this part of the proof is unnecessary.) ($|X|>=2$). Therefore to ensure Equation~\ref{eqn:nc} holds, it is sufficient  
\begin{eqnarray}
    \tilde{n} &\geq& \frac{2}{p_{x_i}} n_{x_i} \ln \frac{2|X|}{\delta} \\
    & = & \frac{2}{p_{x_i}} (8\sigma^2 p_{x_i} \frac{|X|}{\epsilon^2}\ln\frac{4|X||A|}{\delta} \ln \frac{2|X|}{\delta}) \\
    & = & 16\sigma^2  \frac{|X|}{\epsilon^2}\ln\frac{4|X||A|}{\delta} \ln \frac{2|X|}{\delta} \label{eqn:sufficn}
 \end{eqnarray}
 where in the second line,  we have substituted in Equation~\ref{eqn:nx}. Note that this bound on the number of samples $\tilde{n}$ is independent of the particular context $x_i$ we considered, and therefore will hold for all $|X|$ contexts. 

Therefore the probability of any context $x_i$ failing to be observed at least the required $n_{x_i}$ times after $\tilde{n}$ (Equation~\ref{eqn:sufficn}) samples is at most

\begin{equation}
    P(\exists x_i \; \tilde{n}_{x_i} < n_{x_i}  | \tilde{n}) \leq \sum_{i=1}^{|X|} \frac{\delta}{2|X|} = \frac{\delta}{2}
\end{equation}

We now bound the total probability of failure to obtain a near optimal policy as follows:

\begin{eqnarray}
 \P( U(\pistar) - U(\pihat) >  \epsilon)  &=&  \P( U(\pistar) - U(\pihat) >  \epsilon | \tilde{n}_x \geq n_x \forall x ) P(  \tilde{n}_x \geq n_x \forall \;  x ) \\
 & & + \P( U(\pistar) - U(\pihat) >  \epsilon | \exists x \tilde{n}_x < n_x ) P(  \exists x \; \tilde{n}_x < n_x  ) \\
 & \leq & \frac{\delta}{2} \times 1 + 1 \times \frac{\delta}{2} \\
 & \leq & \delta, 
\end{eqnarray}
where in the first inequality we have used the fact that all probabilities are bounded by 1 and substituted in the previous bounds on these failure events. 
\qed
\end{proof}

Now we prove Theorem~\ref{thm:rct_linear}. First, in the following theorem, we prove a result for sample complexity under an arbitrary static (non-adaptive to the observed outcomes) data collection strategy $\pi$. In the following lemma, we then show that we can design a data gathering strategy $\pi$ to achieve a bound that scales roughly like $d^2/\epsilon^2$.

\begin{theorem}[Restatement of Theorem~\ref{thm:rct_linear}]
Assume the reward is linear. 
For any static (non-adaptive to the observed outcomes) strategy $\pi$ used to collect samples,
let 
\begin{align*}
\Sigma(\pi) & = \EE [ \phi(X, \pi(X)) \phi(X, \pi(X))^T] \\
& = \sum_{x,k} \P(X=x) \cdot \P(\pi(x) = a_k) \cdot \phi(x, a_k) \phi(x, a_k)^T
\end{align*}
be the induced covariance matrix.
Define a problem-dependent constant 
$$\rho_0 = \max_{x, k} \norm{\Sigma(\pi)^{-1/2}\phi(x, a_k)}/\sqrt{d}.$$ 
Then, we can design a data collection strategy such that, for any $\epsilon > 0, \delta > 0, \lambda_g \geq 0$
and 
\[n \geq \max\{6 \rho_0^2 d\log(3d/\delta), \BigO{\sigma^2 d^2/\epsilon^2}\} ~ (\text{with cost incurred } c \leq c_{max}n), \]
we have $\P(U(\pistar) - U(\pihat) < \epsilon) > 1- \delta$.
\end{theorem}

\begin{proof}{Proof.}
    Let $\thetahat$ be the linear regression estimator and $\rhat_{xk} = \dotprod{\phi(x, k), \thetahat}$. Then by Theorem 1 of \cite{hsu2014random} we have that for $n \geq 6\rho_0^2d\log\frac{3d}{\delta}$,
    
    \begin{equation} \label{eq:hsu_bound}
        \norm{\thetahat - \thetastar}_{\Sigma(\pi)}^2 \leq \frac{\sigma^2(d + 2\sqrt{d\log\frac{3}{\delta}} + 2\log\frac{3}{\delta})}{n} + o(1/n)
    \end{equation}
    with probability at least $1-\delta$. Now by Lemma \ref{lemma:rew_to_val_err_bound} and Cauchy-Schwarz, we have that 
    \begin{equation}
    \label{eqn:lin}
        U(\pistar) - U(\pihat) \leq 2 \sum_x p_x \max_k |r_{xk} - \rhat_{x k}| \leq 2 \norm{\thetahat - \thetastar}_{\Sigma(\pi)} \sum_x p_x \max_k \norm{\phi(x,k)}_{\Sigma(\pi)^{-1}}.
    \end{equation}
    Let $c(\pi) = \sum_x p_x \max_k \norm{\phi(x,k)}_{\Sigma(\pi)^{-1}}$ be a data-dependent constant.  From Lemma~\ref{lemma:kw}, there exists a data gathering strategy $\tilde{\pi}$ such that $c(\tilde{\pi}) \leq \sqrt{d}$.
    
    Combining this with Equation~\ref{eqn:lin} and Equation~\ref{eq:hsu_bound}, we obtain that if we collect at least
    $$n \geq \maximum{ 6\rho_0^2d\log\frac{3d}{\delta}, \BigO{\frac{\sigma^2d^2}{\epsilon^2}}}$$ under data collection strategy $\pitilde$ then $U(\pistar) - U(\pihat) \leq \epsilon$ with probability at least $1-\delta$.
    
\end{proof}

\begin{lemma}[Modified Kiefer-Wolfowitz Theorem]\label{lemma:kw}
    Let $\Pi = \{\pi \in \Reals^{|X|*K}| \sum_{k}\pi_{xk} = 1, ~ \forall x\}$ be the set of context-conditioned policy distributions.
Then for any context distribution and feature space, we can design a contextual data collection strategy $\tilde{\pi} \in \Pi$ such that 
    \[c(\tilde{\pi}) = \sum_x p_x \max_k \norm{\phi(x, k)}_{\Sigma(\pitilde)^{-1}} \leq \sqrt{d}.\]
\end{lemma}
\begin{proof}{Proof.}
    We adapt the Kiefer-Wolfowitz Theorem concerning G-optimal experimental designs for our setting where we do not have full control over the sampling distribution, but rather can only control the policy distribution (not the context distribution). 
    
    For our proof, define $g(\pi) = \sum_x p_x \max_k \norm{\phi(x, k)}^2_{\Sigma(\pi)^{-1}}$. Our goal will be to show that $\min_{\pi \in \Pi} g(\pi) = g(\pitilde) = d$ and by convexity, 
    \begin{equation}
        c(\pitilde)^2 = (\sum_x p_x \max_k \norm{\phi(x, k)}_{\Sigma(\pitilde)^{-1}})^2 \leq \sum_x p_x \max_k \norm{\phi(x, k)}^2_{\Sigma(\pitilde)^{-1}} = g(\pitilde) \leq d.
    \end{equation}
    To show this we will first optimize $f(\pi) = \log \det \Sigma(\pi)$ and then show that $f(\pi)$ and $g(\pi)$ have the same optimizer $\pitilde$ and that $f(\pitilde)=g(\pitilde)=d$. Note that 
        \begin{align}
            \frac{\partial}{\partial \pi_{xk}} f(\pi) 
                &= \frac{1}{\det \Sigma(\pi)} \frac{\partial}{\partial \pi_{xk}} \det \Sigma(\pi) \\
                &= \textrm{trace}\parens*{\frac{\textrm{adj}(V(\pi))}{\det \Sigma(\pi)}p_x\phi(x, k)\phi(x, k)^T} \\
                &= \textrm{trace}\parens*{\Sigma(\pi)^{-1}p_x\phi(x, k)\phi(x, k)^T} \\
                &= p_x \norm{\phi(x, k)}^2_{\Sigma(\pi)^{-1}}.
        \end{align}
    Since $f$ is concave, by first order optimality conditions, for any $\pi \in \Pi$ and $\pitilde = \argmax_{\pi \in \Pi} f(\pi)$,
    \begin{align}
        0 \geq \dotprod{\nabla f(\pitilde), \pi-\pitilde} &= \sum_x p_x\sum_k \pi_{xk} \norm{\phi(x, k)}^2_{\Sigma(\pitilde)^{-1}} - \sum_x p_x\sum_k \pitilde_{xk} \norm{\phi(x, k)}^2_{\Sigma(\pitilde)^{-1}} \\
        &= \sum_x p_x\sum_k \pi_{xk} \norm{\phi(x, k)}^2_{\Sigma(\pitilde)^{-1}} - d
    \end{align}
    since for any $\pi$,
    \begin{equation}
    \sum_x p_x\sum_k \pi_{xk} \norm{\phi(x, k)}^2_{\Sigma(\pi)^{-1}} = \textrm{trace}(\sum_x p_x\sum_k \pi_{xk} \phi(x, k)\phi(x, k)^T\Sigma(\pi)^{-1}) = \textrm{trace}(I_d) = d.
    \end{equation}
    Thus letting $\pi_{xk} = \indic{k=\argmax_{k'}\norm{\phi(x, k')}_{\Sigma(\pitilde)^{-1}}}$ we have that 
    \begin{equation}
        g(\pitilde) = \sum_x p_x \max_k \norm{\phi(x, k)}^2_{\Sigma(\pitilde)^{-1}} \leq d.
    \end{equation}
    But it also follows that for any $\pi$,
    \begin{equation}
        g(\pi) = \sum_x p_x \max_k \norm{\phi(x, k)}^2_{\Sigma(\pi)^{-1}} \geq \sum_x p_x\sum_k \pi_{xk} \norm{\phi(x, k)}^2_{\Sigma(\pi)^{-1}} = d.
    \end{equation}
    Therefore $\pitilde$ minimizes $g(\pi)$ and $g(\pitilde) = d$.     Q.E.D.
        
\end{proof}

We note that if we know the context distribution we can efficiently solve for $\pistar$ since by the arguments of Lemma~\ref{lemma:kw} we can solve the equivalent optimization problem 
\begin{align}
\optimax{\pi}{\log\det\Sigma(\pi)}{0 \preceq \pi \preceq 1 \notag}
\end{align}
where $\Sigma(\pi) = \sum_x p_x \sum_k \pi_{xk} \phi(x, k) \phi(x, k)^T$. This is an example of a determinant maximizing problem subject to linear matrix constraints, which can be solved efficiently by interior point methods (\cite{maxdet_lmi}).

It is also possible to derive sample complexity bounds when the reward model is a logistic model. To our knowledge, though there has been some work on approaches for cumulative regret minimization under logistic regression, there has not yet been attention to sample complexity bounds for optimal policy estimation in logistic models. Our results here build on work for logistic function estimation but are somewhat restricted, as these results depend on problem-specific constants which depend on the data collection strategy. The provided result does not provide an algorithmic solution and resulting bound, but it does suggest, as do our experiments, that it may be possible to do this quite efficiently.

\begin{theorem}
Assume the reward is logistic, the costs are known, and that the assumptions D0, D1, D2, and C of \cite{ostrovskii2020finitesample} hold (these assumptions define problem-dependent constants $K0, K1, K2, \rho$). Also define $\Sigma$ and $r$ as in Theorem~\ref{thm:rct_linear}. 
Then, for any $\epsilon > 0, \delta > 0, \lambda_g \geq 0$
and 
\[n \geq \BigO{\max\{K_2^4(d + \log\frac{1}{\delta}), \rho K_0^2 K_1^2 d^2 \log\frac{d}{\delta},(\rho^2 r^2 K_1^2 d \log\frac{1}{\delta})/\epsilon^2\}}\]
we have $\P(U(\pistar) - U(\pihat) < \epsilon) > 1- \delta$.

\label{thm:rct_logistic}
\end{theorem}

\begin{proof}{Proof.}
    By Theorem 3.1 of \cite{ostrovskii2020finitesample} (in the well-specified case), for $n \geq \BigO{\max\{K_2^4(d + \log\frac{1}{\delta}), \rho K_0^2 K_1^2 d^2 \log\frac{ed}{\delta}\}}$ with probability at least $1-\delta$,
    \begin{equation}
        \norm{\thetahat_n - \thetastar}^2_H \leq \frac{K_1^2 d\log\frac{e}{\delta}}{n}
    \end{equation}
    where $H = \nabla^2 L_{\pi}(\theta^*)$ is the Hessian of the cross-entropy loss evaluated at the true parameter. By assumption C of \cite{ostrovskii2020finitesample}, we assume that that the covariance matrix $\Sigma = \textrm{Cov}_{\pi}[\phi(X, A)]$ is bounded above by $H$ by a data-dependent factor $\rho$,  i.e. that $\rho H - \Sigma$ is positive semi-definite. Thus by Lemma \ref{lemma:rew_to_val_err_bound},
    \begin{align}
        U(\pistar) - U(\pihat)
            &\leq 2 \sum_x p_x \max_k |r_{xk} - \rhat_{x k}| \\
            &\leq 2 \norm{\thetahat_n - \thetastar}_H \sum_x p_x \max_k \norm{\phi(x, k)}_{H^{-1}} \\
            & \leq 2 \sqrt{\frac{K_1^2 d\log\frac{e}{\delta}}{n}} \sum_x p_x \max_k \rho\norm{\phi(x, k)}_{\Sigma^{-1}}\\
            & \leq 2 r\rho\sqrt{\frac{K_1^2 d\log\frac{e}{\delta}}{n}}.
    \end{align}
    Thus it follows that if $n \geq \BigO{\max\{K_2^4(d + \log\frac{1}{\delta}), \rho K_0^2 K_1^2 d^2 \log\frac{ed}{\delta}, \frac{\rho^2 r^2 K_1^2 d }{\epsilon^2}\log\frac{e}{\delta}\}}$ then $\P(U(\pistar) - U(\pihat) \leq \epsilon) \geq 1-\delta$.     Q.E.D.

\end{proof}

We note that the assumptions D1 and D2 are quite restrictive (as explained in Remark 2.2 of \cite{ostrovskii2020finitesample}). The corresponding constants $K1, K2$ can depend on the magnitude true parameter $\thetastar$ and the data collection policy $\pi$. The authors also note that bounding these constants can be non-trivial, even when the context distribution is known. This makes designing a data collection strategy $\pi$ that minimizes the higher order terms of the derived upper bounds much more difficult than in the linear setting, since this includes $K1, \rho, c$ which all depend on $\pi$.

We now consider bounds on the cost needed to learn a near-optimal policy in the tabular setting. 

\begin{corollary}[Cost Upper Bounds in Tabular Setting]\label{cor:cost_upper_bounds_tabular}
Suppose we take samples according to the same strategy as described in Theorem~\ref{thm:rct_tabular}, but after getting $n_{x k}$ samples for context $x$ and action $k$, we only take the least costly action for that context. Then it is sufficient to spend experimental cost of 

 \begin{align}
        \sum_{x k} c(x, a_k) \frac{8\sigma^2}{\epsilon^2}\log\frac{4|X||A|}{\delta} + n_{extra} \max_{x} \min_{k} c(x, a_k)
    \end{align}
    where 
    \begin{align}
        n_{extra} = \frac{8\sigma^2|A|}{\epsilon^2} \log\frac{4|X||A|}{\delta} \parens*{ \frac{1}{ p_{\text{min}}}\log\parens*{ \frac{16\sigma^2|A|}{\delta \epsilon^2 p_{\text{min}}} \log\frac{4|X||A|}{\delta}} - |X|}
    \end{align}
to learn a policy $\pihat$ such that $\P(U(\pistar) - U(\pihat) < \epsilon) > 1- \delta$.
\end{corollary}

\begin{proof}{Proof.}
    By the arguments of Theorem~\ref{thm:rct_tabular} we need at least $n_{x k} \geq \frac{8\sigma^2}{\epsilon^2}\log\frac{4|X||A|}{\delta}$ samples per context action pair. Thus we will need to incur at least a cost of $\sum_{x k} c(x, a_k) \frac{8\sigma^2}{\epsilon^2}\log\frac{4|X||A|}{\delta}$.
    
    However, as we argued, we need to take additional samples in order to ensure that we get sufficient samples for \textit{every} context action pair. The cost incurred for these additional samples will be random, since we don't know which contexts and actions we will need to sample. In the worst case, we will incur costs $\max_{x} \min_{k} c_{x k}$ for all the remaining samples, and thus the total worst case cost will be 
    
    \begin{align}
        c \leq \sum_{x k} c(x, a_k) \frac{8\sigma^2}{\epsilon^2}\log\frac{4|X||A|}{\delta} + n_{extra} \max_{x} \min_{k} c(x, a_k)
    \end{align}
    where 
    \begin{align}
        n_{extra} = \frac{8\sigma^2|A|}{\epsilon^2} \log\frac{4|X||A|}{\delta} \parens*{ \frac{1}{ p_{\text{min}}}\log\parens*{ \frac{16\sigma^2|A|}{\delta \epsilon^2 p_{\text{min}}} \log\frac{4|X||A|}{\delta}} - |X|}
    \end{align}
    is the number of extra samples that need to be taken beyond the minimum required to ensure at least $n_{x k} \geq \frac{8\sigma^2}{\epsilon^2}\log\frac{4|X||A|}{\delta}$ samples per context-action pair. Q.E.D.
\end{proof}

Note that if there exists a $0$-cost action in every context, then the bound simplifies to $\sum_{x k} c(x, a_k) \frac{8\sigma^2}{\epsilon^2}\log\frac{4|X||A|}{\delta}$. Also if the costs $c$ are uniform across contexts and actions the bound simplifies to $c \parens*{\frac{8\sigma^2 |X| |A|}{\epsilon^2}\log\frac{4|X||A|}{\delta} + n_{extra}}$.

\begin{theorem}[Necessary Cost Bounds in the Tabular Case]\label{thm:lower_cost_bounds_uniform_tabular}
Suppose costs are uniform across contexts and actions, $c(x, a_k) = c, ~\forall x, k$, that the context distribution is uniform so that $\P(X=x) = 1/|X|$, and that the utility function does not consider fairness ($\lambda_g=0$). Then the expected cost of any experimental strategy which learns an $\epsilon$-optimal policy in the sense that $\P(U(\pistar) - U(\pihat) < \epsilon) > 1- \delta$ is at least
$O\parens*{\frac{c|X||A|}{\epsilon^2} \log{\frac{|X|}{\delta}}}$.
\end{theorem}

\begin{proof}{Proof.}
We consider an extension of the hard bandit case considered by \citet{bandit_complexity_lower_bound} in the proof of their Theorem~1. \citet{bandit_complexity_lower_bound} show that to ensure that the right action is chosen with probability at least $1-\delta$ after stopping once $N_{x k}$ samples are taken for every context and action in the hard setting when the reward gaps for each context are of size $\epsilon$, the expected number of samples needed for each context and action is $\Expdist{0}{N_{x, k}} \geq O(1/\epsilon^2 \log{1/\delta})$. 

Consider the environments,
\begin{align}
    H_0: 
        r_{x, a_0} = \frac{1}{2} + \frac{3\epsilon}{2},
        \quad r_{x, a_{k \neq 0}} = \frac{1}{2}, ~\forall x
\end{align}
\begin{align}
    H_k: 
    r_{x, a_0} = \frac{1}{2} + \frac{3\epsilon}{2}, 
    \quad r_{x, a_k} = \frac{1}{2} + 3\epsilon, 
    \quad r_{x, a_{k' \neq 0, k}} = \frac{1}{2} , ~\forall x.
\end{align}

In order to figure out which hypothesis is true and ensure that $\P(U(\pistar) - U(\pihat) < \epsilon) > 1- \delta$ the policy cannot choose the wrong action in more than $2/3$ of the contexts. Thus the total expected cost must be at least 
\begin{align}
    \Expdist{0}{C} = \Expdist{0}{\sum_{k, 2/3 \text{ of contexts } x} c N_{x k}} \geq c\frac{2|X||A|}{3}O\parens*{\frac{4}{9\epsilon^2} \log{2|X|/3\delta}} = O\parens*{\frac{c|X||A|}{\epsilon^2} \log{\frac{|X|}{\delta}}}
\end{align}
where we've applied a union bound to ensure that the right action is learned in \textit{all} of the necessary contexts with probability at least $1-\delta$. Q.E.D.
\end{proof}

Note that in more general settings when the costs and context are not uniform,  cost-efficient sampling is more complicated. For example, it may be most cost effective to prioritize data collection from frequently occurring  contexts with large estimated differences in reward outcomes across actions (and thus with high impact on utility estimation), as well as potential informative context with low sampling costs. When the utility function includes fairness, it is even more difficult to reason about where to prioritize sampling effort, because focusing on the contexts with the highest probability and lowest cost may result in disparities across groups.

\section{Experiment Details for Section~\ref{sec:simulation}}
\label{appx:exp_details}

We define a subpopulation of clients for our simulation from case data at the Santa Clara Public Defender Office according to the following process. 
First, we restrict our population to clients with recorded court dates between January 1, 2010 and November 15, 2021,
who live within 20 miles of the courthouse.
Next, we limit our population to individuals who have stated that their race is white, or that their ethnicity is Vietnamese, or those who have stated that Vietnamese is their preferred language.
We limit to these demographic groups to reflect the motivating example from Section~\ref{sec:motivation}.
Finally, for consistency across case types and differences between court proceedings, we select only the first post-arraignment appearance for all individuals.

Next, we calculate a feature set $x$ for each case describing:
(1) whether the client identifies as Vietnamese;
(2) whether the case is a felony;
(3) whether the client identifies as male;
(4) the client's age;
(5) the natural log of the distance, in miles, between the client's home address and the courthouse (we normalize distance by dividing by 20 miles---the maximum allowable distance---so that the normalized distance is negative, with values of higher magnitude being closer to the courthouse);
(6) the number of known failures to appear in the past two years; and
(7) the inverse number of required court appearances in the past two years.
We further restrict the population to only cases which have complete information on all the above attributes.
The above process results in 12,646 example cases for use in our simulation.

With this information, we model the likelihood a client will appear in court with a logistic regression trained on the above historical population using the stated feature set,
with $\beta$ representing the vector of coefficients corresponding to each feature.
Specifically, we have:
\begin{equation*}
\Pr(Y(0)=1) = \text{logit}^{-1}(X \beta).
\end{equation*}
We modify our simulated population so that the overall population appearance rate is just above 50\%, 
and the appearance rate for Vietnamese clients is just above 70\%.
To do this, we use the model above with the fitted coefficients, but make two modifications. 
First, we change the baseline appearance rate by adjusting the model intercept.
In the historical data we use, appearance rates hover around 90\%, 
which is unusually high compared to other jurisdictions.
For example, in \citet{fishbane2020behavioral}, 
which studied court date attendance in New York City, 
appearance rates were approximately 60\%.
To roughly match this appearance rate,
and increase the potential treatment effect of rides and transit vouchers, 
we set the intercept in our simulation to zero, 
which results in an overall population mean appearance rate of roughly 51\%.
Next, we increase the coefficient for Vietnamese clients from 0.3 in the empirical data to 1 in our simulation,
which sets the average appearance rate for Vietnamese clients at 71\%.
This change---in addition to the fact that Vietnamese individuals tend to live farther away from court---magnifies the tradeoff in spending impact between Vietnamese and white individuals in our simulation, given that treatment effects would tend to be lower, and costs higher, for Vietnamese clients.
These plausible values were chosen so that a reasonable preference for spending parity would be in tension with simply allocating assistance to those with the highest treatment effect per dollar.

The predicted appearance probabilities from the model described above serve as the base for our simulation's structural equation model.
To begin, we define three potential outcomes for each individual, corresponding to 
appearance in the absence of assistance ($k=0$, i.e., that predicted by the above model), 
appearance if provided with rideshare assistance ($k=1$), 
and appearance if provided a transit voucher ($k=2$).
We do so in terms of the following structural equation:
\begin{equation*}
\begin{aligned}
f_Y(k, x, u) = \mathbbm{1}(u \leq \ \text{logit}^{-1}(  x\hat{\beta} ~+ 
 \gamma_1 \cdot \mathbbm{1}(k=1) ~+ 
 \gamma_2 \cdot \mathbbm{1}(k=2) \cdot x_\text{dist} 
))
\end{aligned}
\end{equation*}
where we set $\gamma_1$ to 4 and $\gamma_2$ to -0.75.
Finally, for a latent variable $U_Y \sim \textsc{Unif}(0,1)$, we define the potential outcomes:
\begin{equation*}
    Y(a) = f_Y(a, X, U_Y).
\end{equation*}
This structure ensures that $Y(0) \leq Y(1)$ and that $Y(0) \leq Y(2)$, meaning that receiving any form of assistance is always better than no assistance.
Further, the type of assistance---transit voucher or rideshare assistance---that is best for each individual varies across the population.

As described in the main text, the utility $U$ is defined by Eq.~\eqref{eq:utility}, where we set $r(x, a, y) = y$, $f(x,a,y) = c(x,a)$, and $\lambda_g = 0.0006$.
In other words,
\begin{align*}
U(\pi) &= \EE[Y(\pi(X))] - \sum_{g \in \mathcal{G}} \lambda_g 
             \Big\lvert
             \EE_{X}[c(X, \pi(X)) \mid g \in s(X)] - \EE_{X}[c(X, \pi(X))] 
             \Big\rvert.
\end{align*}
The first term in $U$ is the expected number of clients that would show up under the policy $\pi$, 
and the second term captures our parity preferences.
The constant $\lambda_g$ was chosen 
to reflect a preference for trading off 
between appearance maximization and spending parity. 
For the $\varepsilon$-greedy model, we set $\varepsilon = 0.1$. 
For both UCB and Thompson sampling, we use the default 
weakly informative priors provided by the \texttt{sim} function in \texttt{arm}~\citep{armpackage}.
For UCB, we used the 97.5th percentile estimate of the posterior. 

When estimating $\EE[Y(\pi(a)) \mid X = x]$ during policy learning, we use a logistic regression with the same functional form as the data-generating process above.
We started each of our experiments with a randomly selected warm-up group of 4 people, with at least one male and at least one Vietnamese client. 
During this period, the first two clients were assigned to actions $k=1$ and $k=2$, respectively. 
The other two individuals were assigned to control, i.e., $k=0$.
The treatments during this warm-up period are not included as expenditures against the overall budget $b$.

For our simulation, we set an average per-person budget of \$5.
We also set rideshare costs at \$5 per mile, 
and daily transit voucher costs at \$7.50, reflecting typical prices observed in Santa Clara county.
Because our inferred policies $\pi_i^*$ evolve over time, they are not guaranteed to adhere to the budget constraints.
To account for this possibility,
if we find ourselves spending more on an action than is budgeted, we gradually lower the nominal budget for that action until it meets the target budget (and vice versa for underspending).
Specifically, for each iteration $i$, we compute a new budget $b_{i}^*$:
\begin{equation}
\label{eq:budget_adjustment}
\begin{aligned}
b_{i}^{*} =   b \cdot \frac{b \cdot (i - 1)}{\sum_{j=1}^{i-1} c(x_j, A_j)},
\end{aligned}
\end{equation}
where $A_j$ is the action taken on the 
$j$-th individual,
and $b$ is the target budget.
In Figure~\ref{fig:budget_adherence}, we show that, in expectation, all approaches included in our simulation spend the allowed budget.
In Figure~\ref{fig:cumulative_spending_variation}, we show that across simulations which used RA, UCB, Thompson sampling, or an $\epsilon$-greedy approach, we spent an amount reasonably close to the intended budget.

\begin{figure}[t!]
\begin{center}
\centerline{\includegraphics[width=0.65\columnwidth]{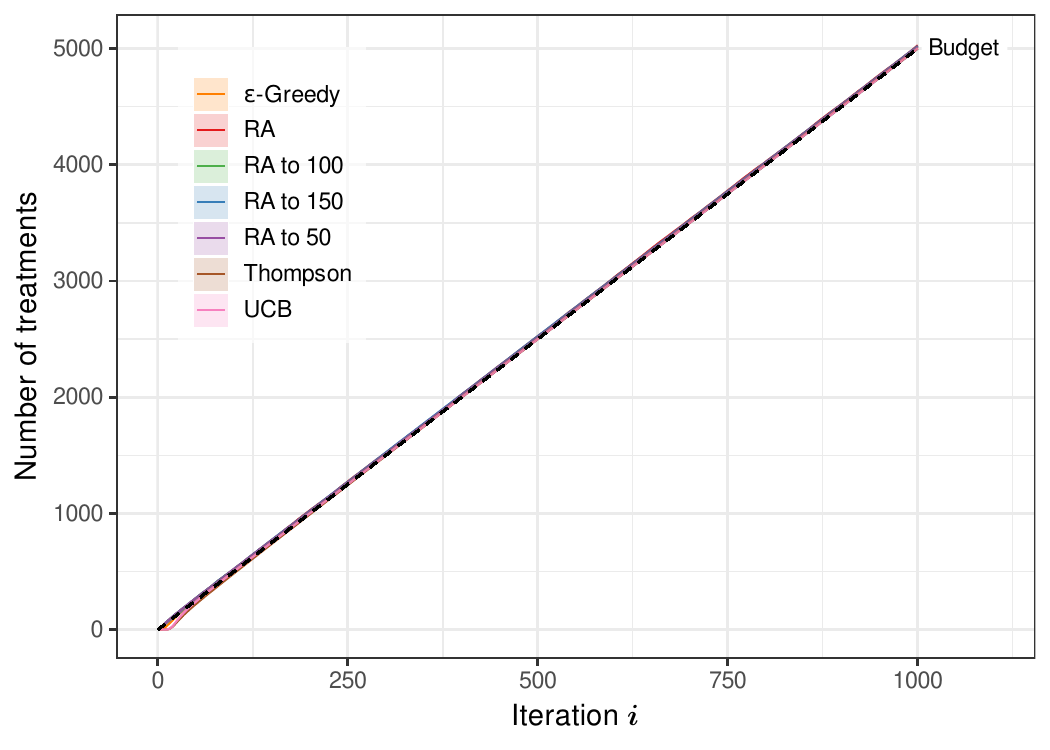}}
\caption{
Mean spending by method across 2,000 simulations. 
The budget is illustrated with a dashed line.
}
\label{fig:budget_adherence}
\end{center}
\end{figure}

\begin{figure}[t!]
\begin{center}
\centerline{\includegraphics[width=0.65\columnwidth]{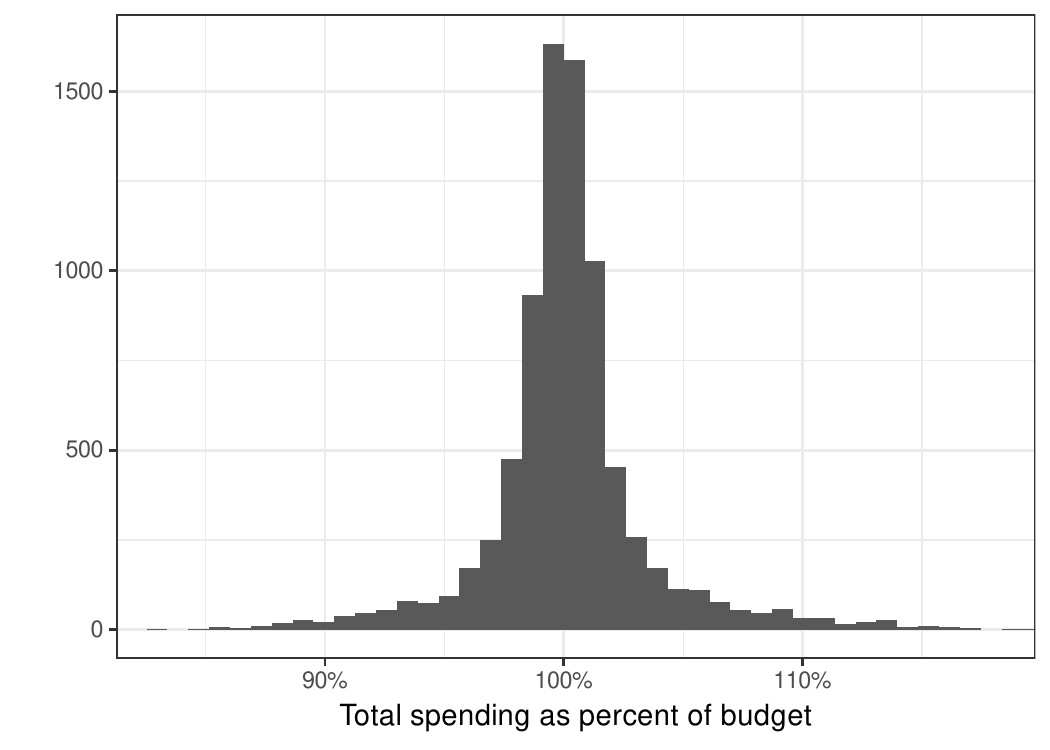}}
\caption{
Distribution of total expenditures, compared to the intended budget,
across experiments which used RA, UCB, Thompson sampling, or an $\epsilon$-greedy approach
to allocate treatments.
80\% of runs spent a total amount under 101.5\% of the total budget,
and 95\% of runs spent a total amount under 105\% of the intended budget.
}
\label{fig:cumulative_spending_variation}
\end{center}
\end{figure}

For our RA and $\epsilon$-greedy approaches, care must be taken when selecting actions given varying costs and an overall per-person budget.
For example, an RA approach that selects all available treatments with equal probability could overshoot the budget if rides cost \$100 on average and the per-person budget is \$5.
To avoid this outcome, we first calculate the expected cost of all actions---including both costly and no-cost actions---when following random allocation, 
and then calculate the proportion $p$ of individuals to whom we can afford to randomly assign an action:
\begin{equation}
\label{eq:random_allocation}
p = \frac{b}{\tilde{c}}, \hspace{0.25cm} \text{where}~\tilde{c} = \frac{1}{k}\sum_k\EE_X[c(X, a_k)].
\end{equation}
Once calculated, we randomly select a proportion $p$ of the population to receive a random action, 
and treat the remainder of the population $1-p$ with the no-cost treatment, 
which ensures we meet our budget in expectation.

Our optimization procedure (i.e., our linear program) formally relies on having a discrete covariate space, but our synthetic population has two continuous covariates: 
the client's age and
distance from the courthouse.
Treating covariates as discrete---for example, by binning all covariates to limit the number of possible values---would still require us to learn a policy across an infeasibly large number of possible values.
For example, if we bin all covariates so that each covariate is represented by no more than ten bins,
we would still need to learn a policy across up to 80,000 possible combinations of covariate values.
To address this issue, we transfer our continuous setting to the discrete setting in two steps.
First, at the start of our experiments, we draw one random sample $\mathcal{C}$ of $n=1{,}000$ clients, and approximate the full population by a discrete distribution over this observed sample, with each client assigned probability $1/n$.
Now, the policies we construct (i.e., those produced by our LP) are technically defined only for individuals having covariates matching those of a client in the initial sample $\mathcal{C}$. 
Consequently, when making decisions for a new individual with covariates $x$, we act according to the learned policy for the most similar client in $\mathcal{C}$, among those having the same group membership $s(x)$ as the new client.
Specifically, let $\hat{h}(x,k)$ be our estimate of
$\EE[Y(\pi(a)) \mid X = x]$.
Then,
for a new client, we define its nearest neighbor $\textsc{NN}(x)$ to be:
\begin{equation*}
    \textsc{NN}(x) = 
    \argmin_{\genfrac..{0pt}{3}{x' \in \mathcal{C}}{s(x') = s(x)}}
    \Bigg\lVert \frac{\hat{h}(x', \cdot)}{c(x', .)} - \frac{\hat{h}(x, \cdot)}{c(x, .)} \Bigg\rVert_2.
\end{equation*}
Then, for any policy $\pi$ defined on $\mathcal{C}$, we extend it to a policy $\tilde{\pi}$ on the full population by setting $\tilde{\pi}(x) = \pi(\textsc{NN}(x))$.

This procedure is insensitive to the number of randomly sampled clients in $\mathcal{C}$.
We compared cumulative observed utility 
when randomly sampling 1,000 clients for $\mathcal{C}$ against cumulative observed utility when randomly sampling 500 and 2,000 clients.
Across 400 simulations of the UCB approach, 
the cumulative observed utility at the end of the experiment
for both the 500 and 2,000 client variations
was no more than 0.03\% different 
than the cumulative observed utility when $\mathcal{C}$ was composed of 1,000 clients.

To construct Figure~\ref{fig:spending_variations}, 
we conducted a grid search over different possible random assignments,
varying allocations by 10\% at a time, 
allowing allocations for each costly arm to vary from 0\% to 100\%.
We also explored random assignments 
with a lower proportion of costly treatments,
varying allocations by 2\% at a time and
allowing allocations for costly arms to vary from 2\% to 8\%.
We ran 125 simulations of each allocation variation.
Aside from the allocation variations set by grid search, 
all experiment parameters were otherwise identical to those described in the main simulation 
(e.g., each experiment included 1,000 clients).

\end{document}